\documentclass{article}

\usepackage{microtype}
\usepackage{graphicx}
\usepackage{subfigure}
\usepackage{booktabs} 
\usepackage{amsfonts,amsmath,amsthm,amssymb}
\usepackage{multirow}
\usepackage{enumitem}
\usepackage{url}

\newtheorem{theo}{Theorem}

\newtheorem{lemma}{Lemma}
\newtheorem{corollary}{Corollary}
\DeclareMathOperator*{\argmin}{arg\,min}
\DeclareMathOperator*{\argmax}{arg\,max}

\usepackage{hyperref}

\usepackage[accepted]{icml2021}

\icmltitlerunning{Pointwise Binary Classification with Pairwise Confidence Comparisons}

\begin{document}

\twocolumn[
\icmltitle{Pointwise Binary Classification with Pairwise Confidence Comparisons}




\begin{icmlauthorlist}
\icmlauthor{Lei Feng}{first}
\icmlauthor{Senlin Shu}{second}
\icmlauthor{Nan Lu}{third,fourth}
\icmlauthor{Bo Han}{fifth,fourth}
\icmlauthor{Miao Xu}{sixth,fourth}
\icmlauthor{Gang Niu}{fourth}
\icmlauthor{Bo An}{seventh}
\icmlauthor{Masashi Sugiyama}{fourth,third}
\end{icmlauthorlist}

\icmlaffiliation{first}{College of Computer Science, Chongqing University, China}
\icmlaffiliation{second}{College of Computer and Information Science, Southwest University, China}
\icmlaffiliation{third}{The University of Tokyo, Japan}
\icmlaffiliation{fourth}{RIKEN Center for Advanced Intelligence Project, Japan}
\icmlaffiliation{fifth}{Department of Computer Science, Hong Kong Baptist University, China}
\icmlaffiliation{sixth}{School of Information Technology and Electrical Engineering,
The University of Queensland, Australia}
\icmlaffiliation{seventh}{School of Computer Science and Engineering, Nanyang Technological University, Singapore}

\icmlcorrespondingauthor{Lei Feng}{lfeng@cqu.edu.cn}

\icmlkeywords{Machine Learning, ICML}

\vskip 0.3in
]



\printAffiliationsAndNotice{}  

\begin{abstract}
To alleviate the data requirement for training effective binary classifiers in binary classification, many \emph{weakly supervised learning} settings have been proposed.
Among them, some consider using \emph{pairwise} but not \emph{pointwise} labels, when pointwise labels are not accessible due to privacy, confidentiality, or security reasons.
However, as a pairwise label denotes whether or not two data points share a pointwise label, it cannot be easily collected if either point is equally likely to be positive or negative.
Thus, in this paper, we propose a novel setting called \emph{pairwise comparison (Pcomp) classification}, where we have only pairs of unlabeled data that we know one is more likely to be positive than the other.
Firstly, we give a Pcomp \emph{data generation process}, derive an \emph{unbiased risk estimator}~(URE) with theoretical guarantee, and further improve URE using \emph{correction functions}.
Secondly, we link Pcomp classification to \emph{noisy-label learning} to develop a progressive URE and improve it by imposing \emph{consistency regularization}.
Finally, we demonstrate by experiments the effectiveness of our methods, which suggests Pcomp is a valuable and practically useful type of pairwise supervision besides the pairwise label.
\end{abstract}

\section{Introduction}
Traditional supervised learning techniques have achieved great advances, while they require precisely labeled data. In many real-world scenarios, it may be too difficult to collect such data. To alleviate this issue, a large number of weakly supervised learning problems \cite{zhou2018brief} have been extensively studied, including semi-supervised learning \cite{zhu2009introduction,niu2013squared,sakai2018semi}, multi-instance learning \cite{zhou2009multi,sun2016multiple,zhang2017multi}, noisy-label learning \cite{han2018co,xia2019anchor,wei2020combating}, partial-label learning \cite{zhang2017disambiguation,feng2020provably,lv2020progressive}, complementary-label learning \cite{ishida2017learning,yu2018learning,Ishida2019Complementary,feng2020learning}, positive-unlabeled classification \cite{elkan2008learning,niu2016theoretical,gong2019loss,chen2020self}, positive-confidence classification \cite{Ishida2018Binary}, similar-unlabeled classification \cite{Bao2018Classification},
similar-dissimilar classification \cite{shimada2020classification,bao2020similarity}, unlabeled-unlabeled classification \cite{lu2019on,lu2020mitigating}, and triplet classification \cite{cui2020classification}.

Among these weakly supervised learning problems, some of them \cite{Bao2018Classification,shimada2020classification,bao2020similarity} consider learning a binary classifier with \emph{pairwise} labels that indicate whether two instances belong to (similar) the same class or not (dissimilar), when \emph{pointwise} labels are not accessible due to privacy, confidentiality, or security reasons. However, if either of the two instances is equally likely to be positive or negative, it becomes difficult for us to accurately collect the underlying pairwise label of them.
This motivates us to consider using another type of pairwise supervision (instead of the pairwise label) for successfully learning a binary classifier.

In this paper, we propose a novel setting called \emph{pairwise comparison (Pcomp) classification}, where we aim to perform pointwise binary classification with only \emph{pairwise comparison data}.
A pairwise comparison $(\boldsymbol{x},\boldsymbol{x}^\prime)$ represents that the instance $\boldsymbol{x}$ has a larger confidence of belonging to the positive class than the instance $\boldsymbol{x}^\prime$. Such weak supervision (pairwise confidence comparison) could be much easier for people to collect than full supervision (pointwise label) in practice, especially for applications on sensitive or private matters. 
For example, it may be difficult to collect sensitive or private data with pointwise labels, as asking for the true labels could be prohibited or illegal.
In this case, it could be easier for people to collect other weak supervision like the comparison information between two examples. 

It is also advantageous to consider pairwise confidence comparisons in pointwise binary classification with class overlapping, where the labeling task is difficult, and even experienced labelers may provide wrong pointwise labels. Let us denote the labeling standard of a labeler as $\tilde{p}(y|\boldsymbol{x})$ and assume that an instance $\boldsymbol{x}_1$ is more positive than another instance $\boldsymbol{x}_2$. Facing the difficult labeling task, different labelers may hold different labeling standards, $\tilde{p}(y=+1|\boldsymbol{x}_1)>\tilde{p}(y=+1|\boldsymbol{x}_2)>1/2$, $\tilde{p}(y=+1|\boldsymbol{x}_1)>1/2>\tilde{p}(y=+1|\boldsymbol{x}_2)$, and $1/2>\tilde{p}(y=+1|\boldsymbol{x}_1)>\tilde{p}(y=+1|\boldsymbol{x}_2)$, thereby providing different pointwise labels: $(+1,+1)$, $(+1,-1)$, $(-1,-1)$. We can find that different labelers may provide inconsistent pointwise labels, while pairwise confidence comparisons are unanimous and accurate. 
One may argue that we could aggregate multiple labels of the same instance using crowdsourcing learning methods \cite{whitehill2009whose,raykar2010learning}. However, as not every instance will be labeled by multiple labelers, it is not always applicable to crowdsourcing learning methods. Therefore, our proposed Pcomp classification is useful in this case. 

Our main contributions can be summarized as follows:
\begin{itemize}[leftmargin=0.5cm]
\vspace{-10pt}
\item We propose \emph{pairwise comparison (Pcomp) classification}, a novel weakly supervised learning setting, and present a mathematical formulation for the generation process of pairwise comparison data.
\vspace{-5pt}
\item We prove that an \emph{unbiased risk estimator} (URE) can be derived, propose an \emph{empirical risk minimization} (ERM) based method, and present an improvement using correction functions \cite{lu2020mitigating} for alleviating overftting when complex models are used.
\vspace{-5pt}
\item We start from the noisy-label learning perspective to introduce the \emph{RankPruning} method \cite{northcutt2017learning} that holds a \emph{progressive} URE for solving our proposed Pcomp classification problem and improve it by imposing \emph{consistency regularization}. 
\vspace{-10pt}
\end{itemize}
Extensive experimental results demonstrate the effectiveness of our proposed solutions for Pcomp classification.

\section{Preliminaries}
Binary classification with pairwise comparisons and extra pointwise labels has been studied \cite{xu2017noise,kane2017active}, while our work focuses on a more challenging problem where \emph{only pairwise comparison examples are provided}. To the best of our knowledge, we are the first to investigate such a challenging problem.
Unlike previous studies \cite{xu2017noise,kane2017active} that leverage some pointwise labels to differentiate the labels of pairwise comparisons, our methods are purely based on ERM with only pairwise comparisons. In the next, we briefly introduce some notations and review related problem formulations.

\noindent\textbf{Binary Classification.}\quad Since our paper focuses on how to train a binary classifier from pairwise comparison data, we first review the problem formulation of binary classification. Let the feature space be $\mathcal{X}$ and the label space be $\mathcal{Y}=\{+1,-1\}$. Suppose the collected dataset is denoted by $\mathcal{D}=\{(\boldsymbol{x}_i,y_i)\}_{i=1}^n$ where each example $(\boldsymbol{x}_i,y_i)$ is independently sampled from the joint distribution with density $p(\boldsymbol{x},y)$, which includes an instance $\boldsymbol{x}_i\in\mathcal{X}$ and a label $y_i\in\mathcal{Y}$. The goal of binary classification is to train an optimal classifier $f:\mathcal{X}\mapsto\mathbb{R}$ by minimizing the following (expected) classification risk:
\begin{align}
\nonumber
R(f) &= \mathbb{E}_{p(\boldsymbol{x},y)}\big[\ell(f(\boldsymbol{x}),y)\big]\\
\nonumber
&=\pi_+\mathbb{E}_{p_+(\boldsymbol{x})}\big[\ell(f(\boldsymbol{x}),+1)\big]\\
\label{ori_risk}
&\quad\quad\quad\quad\quad\quad\quad+\pi_-\mathbb{E}_{p_-(\boldsymbol{x})}\big[\ell(f(\boldsymbol{x}),-1)\big],
\end{align}
where $\ell: \mathbb{R}\times\mathcal{Y}\mapsto\mathbb{R}_+$ denotes a binary loss function, $\pi_+ := p(y=+1)$ (or $\pi_- := p(y=-1)$) denotes the \emph{positive} (or \emph{negative}) \emph{class prior probability}, and $p_+(\boldsymbol{x}):=p(\boldsymbol{x}|y=+1)$ (or $p_-(\boldsymbol{x}):=p(\boldsymbol{x}|y=-1)$) denotes the \emph{class-conditional probability density} of the positive (or negative) data. ERM approximates the expectations over $p_+(x)$ and $p_-(x)$ by the empirical averages of positive and negative data and the empirical risk is minimized with respect to the classifier $f$.

\noindent\textbf{Positive-Unlabeled (PU) Classification.}\quad In some real-world scenarios, it may be difficult to collect negative data, and only positive (P) and unlabeled (U) data are available. PU classification aims to train an effective binary classifier in this weakly supervised setting. Previous studies \cite{Plessis2014Analysis,Plessis2015Convex,Kiryo2017Positive} showed that the classification risk $R(f)$ in Eq. (\ref{ori_risk}) can be rewritten only in terms of positive and unlabeled data as
\begin{align}
\nonumber
&R_{\mathrm{PU}}(f) = \pi_+\mathbb{E}_{p_+(\boldsymbol{x})}\big[\ell(f(\boldsymbol{x}),+1)-\ell(f(\boldsymbol{x}),-1)\big]\\
\label{pu_estimator}
&\quad\quad\quad\quad\quad\quad\quad\quad\quad\quad\quad \ \ +\mathbb{E}_{p(\boldsymbol{x})}\big[\ell(f(\boldsymbol{x}),-1)\big],
\end{align}
where $p(\boldsymbol{x}) = \pi_+p_+(\boldsymbol{x})+\pi_-p_-(\boldsymbol{x})$ denotes the density of unlabeled data. This risk expression immediately allows us to employ ERM in terms of positive and unlabeled data.

\noindent\textbf{Unlabeled-Unlabeled (UU) Classification.}\quad The recent studies \cite{lu2019on,lu2020mitigating} showed that it is possible to train a binary classifier only from two unlabeled datasets with different class priors. \citet{lu2019on} showed that the classification risk $R(f)$ can be rewritten as
\begin{align}
\nonumber
R_{\mathrm{UU}}(f)&= \mathbb{E}_{p_{\mathrm{tr}}(\boldsymbol{x})}\Big[\frac{(1-\theta^\prime)\pi_+}{\theta-\theta^\prime}\ell(f(\boldsymbol{x}),+1) \\
\nonumber
&\quad\quad\quad\quad\quad\quad -\frac{\theta^\prime(1-\pi_+)}{\theta-\theta^\prime}\ell(f(\boldsymbol{x}),-1)\Big] \\
\nonumber
& + \mathbb{E}_{p_{\mathrm{tr}^\prime}(\boldsymbol{x}^\prime)}\Big[\frac{\theta(1-\pi_+)}{\theta-\theta^\prime}\ell(f(\boldsymbol{x}^\prime),-1) \\
\label{uu_estimator}
&\quad\quad\quad\quad\quad\quad
-\frac{(1-\theta)\pi_+}{\theta-\theta^\prime}\ell(f(\boldsymbol{x}^\prime),+1)\Big],
\end{align}
where $\theta$ and $\theta^\prime$ are different class priors of two unlabeled datasets, and $p_{\mathrm{tr}}(\boldsymbol{x})$ and $p_{\mathrm{tr}^\prime}(\boldsymbol{x}^\prime)$ are the densities of two datasets of unlabeled data, respectively. This risk expression immediately allows us to employ ERM only from two sets of unlabeled data. For $R_{\mathrm{UU}}(f)$ in Eq.~(\ref{uu_estimator}), if we set $\theta=1$, $\theta^\prime=\pi_+$, and replace $p_{\mathrm{tr}}(\boldsymbol{x})$ and $p_{\mathrm{tr}^\prime}(\boldsymbol{x}^\prime)$ by $p_+(\boldsymbol{x})$ and $p(\boldsymbol{x})$ respectively, then we can recover $R_{\mathrm{PU}}(f)$ in Eq.~(\ref{pu_estimator}). Therefore, UU classification could be taken as a generalized framework of PU classification in terms of URE. Besides, Eq.~(\ref{uu_estimator}) also recovers a complicated URE of similar-unlabeled classification \cite{Bao2018Classification} by setting $\theta=\pi_+$ and $\theta^\prime=\pi_+^2/(2\pi_+^2-2\pi_++1)$.

To solve our proposed Pcomp classification problem, we will present a mathematical formulation for the generation process of pairwise comparison data, based on which we will explore two UREs that are compatible with any model and optimizer to train a binary classifier by ERM and establish the corresponding \emph{estimation error bounds}.

\section{Data Generation Process}
In order to derive UREs for performing ERM, we first formulate the underlying generation process of pairwise comparison data\footnote{In contrast to \citet{xu2019uncoupled} and \citet{xu2020regression} that utilized pairwise comparison data to solve the regression problem, we focus on binary classification.}, which consists of pairs of unlabeled data that we know which one is more likely to be positive.
Suppose the provided dataset is denoted by $\widetilde{\mathcal{D}} = \{(\boldsymbol{x}_i,\boldsymbol{x}_i^\prime)\}_{i=1}^n$ where $(\boldsymbol{x}_i,\boldsymbol{x}_i^\prime)$ (with their unknown true labels ($y_i$, $y_i^\prime$)) is expected to satisfy $p(y_i=+1|\boldsymbol{x}_i)> p(y_i^\prime=+1|\boldsymbol{x}_i^\prime)$. 

It is clear that we could easily collect pairwise comparison data if the positive confidence (i.e., $p(y=+1|\boldsymbol{x})$) of each instance could be obtained. However, such information is much harder to obtain than class labels in real-world scenarios. Therefore, unlike some studies \cite{Ishida2018Binary,shinoda2020binary} that assume the positive confidence of each instance is provided by the labeler, we only assume that the labeler has access to the labels of training data. Specifically, we adopt the assumption \cite{cui2020classification} that weakly supervised examples are first sampled from the true data distribution, but the labels are only accessible to the labeler. Then, the labeler would provide us weakly supervised information (i.e., pairwise comparison information) according to the labels of sampled data pairs. That is, for any pair of unlabeled data $(\boldsymbol{x},\boldsymbol{x}^\prime)$, the labeler would tell us whether $(\boldsymbol{x},\boldsymbol{x}^\prime)$ could be collected as a pairwise comparison for Pcomp classification, based on the labels $(y,y^\prime)$ rather than the positive confidences $(p(y=+1|\boldsymbol{x}), p(y=+1|\boldsymbol{x}^\prime))$.

Now, the question becomes: how does the labeler consider $(\boldsymbol{x},\boldsymbol{x}^\prime)$ as a pairwise comparison for Pcomp classification, in terms of the labels $(y,y^\prime)$? As shown in our previous example of binary classification with class overlapping, we could infer that the labels $(y,y^\prime)$ of our required pairwise comparison data $(\boldsymbol{x},\boldsymbol{x}^\prime)$ for Pcomp classification can only be one of the three cases $\{(+1,-1), (+1,+1), (-1,-1)\}$, because the condition $p(y=+1|\boldsymbol{x})\geq p(y^\prime=+1|\boldsymbol{x}^\prime)$ is definitely violated if $(y,y^\prime)=(-1,+1)$. Therefore, we assume that the labeler would take $(\boldsymbol{x},\boldsymbol{x}^\prime)$ as a pairwise comparison example in the dataset $\widetilde{\mathcal{D}}$, if the labels $(y,y^\prime)$ of $(\boldsymbol{x},\boldsymbol{x}^\prime)$ belong to the above three cases. It is also worth noting that for a pair of data $(\boldsymbol{x},\boldsymbol{x}^\prime)$ with labels $(y,y^\prime)=(-1,+1)$,
the labeler could take $(\boldsymbol{x}^\prime,\boldsymbol{x})$ as a pairwise comparison example. Because by exchanging the positions of $(\boldsymbol{x},\boldsymbol{x}^\prime)$, $(\boldsymbol{x}^\prime, \boldsymbol{x})$ could be associated with labels $(+1, -1)$, which belong to the three cases. In summary, we assume that pairwise comparison data are sampled from those pairs of data whose labels belong to the three cases $\{(+1,-1), (+1,+1), (-1,-1)\}$. Based on the above described generation process of pairwise comparison data, we have the following theorem.

\begin{theo}
\label{data_generation}
According to the generation process of pairwise comparison data described above, let
\begin{gather}
\widetilde{p}(\boldsymbol{x},\boldsymbol{x}^\prime) = \frac{q(\boldsymbol{x},\boldsymbol{x}^\prime)}{\pi_+^2+\pi_-^2+\pi_+\pi_-},
\end{gather}
where
$q(\boldsymbol{x},\boldsymbol{x}^\prime)=\pi_+^2p_+(\boldsymbol{x})p_+(\boldsymbol{x}^\prime) + \pi_-^2p_-(\boldsymbol{x})p_-(\boldsymbol{x}^\prime)
+\pi_+\pi_-p_+(\boldsymbol{x})p_-(\boldsymbol{x}^\prime)$.
Then, we can conclude that the collected pairwise comparison data are independently drawn from $\widetilde{p}(\boldsymbol{x},\boldsymbol{x}^\prime)$, i.e.,
$\widetilde{\mathcal{D}} = \{(\boldsymbol{x}_i,\boldsymbol{x}_i^\prime)\}_{i=1}^n\stackrel{\mathrm{i.i.d.}}{\sim}\widetilde{p}(\boldsymbol{x},\boldsymbol{x}^\prime)$.
\end{theo}
The proof is provided in Appendix A. Theorem \ref{data_generation} provides an explicit expression of the probability density of pairwise comparison data.

Next, we would like to extract pointwise information from pairwise information, since our goal is to perform pointwise binary classification.
Let us denote the pointwise data collected from $\widetilde{\mathcal{D}}=\{(\boldsymbol{x}_i,\boldsymbol{x}_i^\prime)\}_{i=1}^n$ by breaking the pairwise comparison relation as $\widetilde{\mathcal{D}}_+=\{\boldsymbol{x}_i\}_{i=1}^n$ and $\widetilde{\mathcal{D}}_-= \{\boldsymbol{x}_i^\prime\}_{i=1}^n$. Then we can obtain the following theorem.
\begin{theo}
\label{relation}
Pointwise examples in $\widetilde{\mathcal{D}}_+$ and $\widetilde{\mathcal{D}}_-$ are independently drawn from $\widetilde{p}_+(\boldsymbol{x})$ and $\widetilde{p}_-(\boldsymbol{x}^\prime)$, where
\begin{align}
\nonumber
\widetilde{p}_+(\boldsymbol{x}) &= \frac{\pi_+}{\pi_-^2+\pi_+}p_+(\boldsymbol{x}) + \frac{\pi_-^2}{\pi_-^2+\pi_+}p_-(\boldsymbol{x}),\\
\nonumber
\widetilde{p}_-(\boldsymbol{x}^\prime) &= \frac{\pi_+^2}{\pi_+^2+\pi_-}p_+(\boldsymbol{x}^\prime) + \frac{\pi_-}{\pi_+^2+\pi_-}p_-(\boldsymbol{x}^\prime).
\end{align}
\end{theo}
The proof is provided in Appendix B. Theorem \ref{relation} shows the relationships between the pointwise densities and the class-conditional densities. Besides, it indicates that from pairwise comparison data, we can essentially obtain examples that are independently drawn from $\widetilde{p}_+(\boldsymbol{x})$ and $\widetilde{p}_-(\boldsymbol{x}^\prime)$.

\section{The Proposed Methods}
In this section, we explore two UREs to train a binary classifier by ERM from only pairwise comparison data with the above generation process.
\subsection{Corrected Pcomp Classification}
In Eq.~(\ref{ori_risk}), the classification risk $R(f)$ could be separately expressed as the expectations over $p_+(\boldsymbol{x})$ and $p_-(\boldsymbol{x})$. Although we do not have access to the two class-conditional densities $p_+(\boldsymbol{x})$ and $p_-(\boldsymbol{x})$, we can represent them by our introduced pointwise densities $\widetilde{p}_+(\boldsymbol{x})$ and $\widetilde{p}_-(\boldsymbol{x})$.
\begin{lemma}
\label{another_relation}
We can express $p_+(\boldsymbol{x})$ and $p_-(\boldsymbol{x})$ in terms of $\widetilde{p}_+(\boldsymbol{x})$ and $\widetilde{p}_+(\boldsymbol{x})$ as
\begin{align}
\nonumber
p_+(\boldsymbol{x})
&= \frac{1}{\pi_+}\big(\widetilde{p}_+(\boldsymbol{x}) - \pi_-\widetilde{p}_-(\boldsymbol{x})\big),\\
\nonumber
p_-(\boldsymbol{x})
&= \frac{1}{\pi_-}\big(\widetilde{p}_-(\boldsymbol{x}) - \pi_+\widetilde{p}_+(\boldsymbol{x})\big).
\end{align}
\end{lemma}
The proof is provided in Appendix C. As a result of Lemma \ref{another_relation}, we can express the classification risk $R(f)$ using only pairwise comparison data sampled from $\widetilde{p}_+(\boldsymbol{x})$ and $\widetilde{p}_-(\boldsymbol{x})$.
\begin{theo}
\label{pc_estimator}
The classification risk $R(f)$ can be equivalently expressed as
\begin{align}
\label{pc_expected_estimator}
R_{\mathrm{PC}}(f) &= \mathbb{E}_{\widetilde{p}_+(\boldsymbol{x})}\big[\ell(f(\boldsymbol{x}),+1)-\pi_+\ell(f(\boldsymbol{x}),-1)\big]\\
\nonumber
&+ \mathbb{E}_{\widetilde{p}_-(\boldsymbol{x}^\prime)}\big[\ell(f(\boldsymbol{x}^\prime),-1)-\pi_-\ell(f(\boldsymbol{x}^\prime),+1)\big].
\end{align}
\end{theo}
The proof is provided in Appendix D. In this way, we could train a binary classifier by minimizing the following empirical approximation of $R_{\mathrm{PC}}(f)$:
\begin{align}
\label{pc_empirical_estimator}
&\widehat{R}_{\mathrm{PC}}(f) = \frac{1}{n}\sum\nolimits_{i=1}^n\big(\ell(f(\boldsymbol{x}_i),+1)+ \ell(f(\boldsymbol{x}_i^\prime),-1) \\
\nonumber
&\quad\quad\quad\quad\quad\quad\quad -\pi_+\ell(f(\boldsymbol{x}_i),-1) -\pi_-\ell(f(\boldsymbol{x}_i^\prime),+1)\big).
\end{align}
\noindent{\textbf{Estimation Error Bound.}}\quad
Here, we establish an estimation error bound for the proposed URE. Let $\mathcal{F}=\{f: \mathcal{X}\mapsto \mathbb{R}\}$ be the model class, $\widehat{f}_{\mathrm{PC}} = \argmin_{f\in\mathcal{F}}\widehat{R}_{\mathrm{PC}}(f)$ be the empirical risk minimizer, and $f^\star=\argmin_{f\in\mathcal{F}}R(f)$ be the true risk minimizer. Let $\widetilde{\mathfrak{R}}_n^+(\mathcal{F})$ and $\widetilde{\mathfrak{R}}_n^-(\mathcal{F})$ be the \emph{Rademacher complexities} \cite{bartlett2002rademacher} of $\mathcal{F}$ with sample size $n$ over $\widetilde{p}_+(\boldsymbol{x})$ and $\widetilde{p}_-(\boldsymbol{x})$ respectively. 
\begin{theo}
\label{estimation_error}
Suppose the loss function $\ell$ is $\rho$-Lipschitz with respect to the first argument ($0\leq\rho\leq\infty$), and all functions in the model class $\mathcal{F}$ are bounded, i.e., there exists a positive constant $C_{\mathrm{b}}$ such that $\left\|f\right\|\leq C_{\mathrm{b}}$ for any $f\in\mathcal{F}$. Let $C_{\ell}:=\sup_{z\leq C_{\mathrm{b}}, t=\pm 1}\ell(z, t)$. Then for any $\delta>0$, with probability at least $1-\delta$, we have
\begin{align}
\nonumber
\textstyle
R(\widehat{f}_{\mathrm{PC}})-R(f^\star)&\textstyle\leq (1+\pi_+)4\rho\widetilde{\mathfrak{R}}_n^+(\mathcal{F}) \\
\nonumber
&\textstyle+(1+\pi_-)4\rho\widetilde{\mathfrak{R}}_n^-(\mathcal{F}) + 6C_{\ell}\sqrt{\frac{\log\frac{8}{\delta}}{2n}}.
\end{align}
\end{theo}
The proof is provided in Appendix E. Theorem \ref{estimation_error} shows that our proposed method is consistent, i.e., as $n\rightarrow\infty$, $R(\widehat{f}_{\mathrm{PC}})\rightarrow R(f^\star)$, since $\widetilde{\mathfrak{R}}_n^+(\mathcal{F})$,  $\widetilde{\mathfrak{R}}_n^-(\mathcal{F})\rightarrow 0$ for all parametric models with a bounded norm such as deep neural networks trained with weight decay \cite{golowich2017size,lu2019on}. Besides, $\widetilde{\mathfrak{R}}_n^+(\mathcal{F})$ and  $\widetilde{\mathfrak{R}}_n^-(\mathcal{F})$ can be normally bounded by $C_{\mathcal{F}}/\sqrt{n}$ for a positive constant $C_{\mathcal{F}}$. Hence, we can further see that the convergence rate is $\mathcal{O}_p(1/\sqrt{n})$ where $\mathcal{O}_p$ denotes the order in probability. This order is the optimal parametric rate for ERM without additional assumptions \cite{mendelson2008lower}.

\noindent{\textbf{Relation to UU Classification}.}\quad It is worth noting that the URE of UU classification $R_{\mathrm{UU}}(f)$ is quite general for binary classification with weak supervision. Hence we also would like to show the relationships between our proposed estimator $R_{\mathrm{PC}}(f)$ and $R_{\mathrm{UU}}(f)$. We demonstrate by the following corollary that under some conditions, $R_{\mathrm{UU}}(f)$ is equivalent to $R_{\mathrm{PC}}(f)$.
\begin{corollary}
\label{coro1}
By setting $p_{\mathrm{tr}} = \widetilde{p}_+(\boldsymbol{x})$, $p_{\mathrm{tr}}^\prime=\widetilde{p}_-(\boldsymbol{x})$, 
$\theta=\pi_+/(1-\pi_+ + \pi_+^2)$, and $\theta^\prime=\pi_+^2/(1-\pi_+ + \pi_+^2)$, Eq.~(\ref{uu_estimator}) is equivalent to Eq.~(\ref{pc_expected_estimator}), which means that $R_{\mathrm{UU}}(f)$ is equivalent to $R_{\mathrm{PC}}(f)$.
\end{corollary}
We omit the proof of Corollary \ref{coro1}, since it is straightforward to derive Eq.~(\ref{pc_expected_estimator}) from Eq.~(\ref{uu_estimator}) with required notations. 

\noindent\textbf{Empirical Risk Correction.}\quad As shown by \citet{lu2020mitigating}, directly minimizing $\widehat{R}_{\mathrm{PC}}(f)$ would suffer from overfitting when complex models are used due to the negative risk issue. More specifically, since negative terms are included in Eq.~(\ref{pc_empirical_estimator}), the empirical risk can be negative even though the original true risk can never be negative. To ease this problem, they wrapped the terms in $\widehat{R}_{\mathrm{UU}}(f)$ that cause a negative empirical risk by certain \emph{consistent correction functions} defined in \citet{lu2020mitigating},
such as the rectified linear unit (ReLU) function $g(z)=\max(0,z)$ and absolute value function $g(z)=|z|$. These consistent correction functions could also be applied to $\widehat{R}_{\mathrm{PC}}$ for alleviating overfitting when complex models are used. In this way, we could obtain the following corrected empirical risk estimator:
\begin{align}
\nonumber
&\widehat{R}_{\mathrm{cPC}}(f) = g\Big(\frac{1}{n}\sum_{i=1}^n\big(\ell(f(\boldsymbol{x}_i),+1)-\pi_-\ell(f(\boldsymbol{x}_i^\prime),+1)\big)\Big)\\
\label{corrected_empirical_risk}
&\quad\ \ +g\Big(\frac{1}{n}\sum\limits_{i=1}^n\big(\ell(f(\boldsymbol{x}_i^\prime),-1)-\pi_+\ell(f(\boldsymbol{x}_i),-1)\big)\Big).
\end{align}
\subsection{Progressive Pcomp Classification}
Here, we start from the noisy-label learning perspective to solve the Pcomp classification problem. Intuitively, we could simply perform binary classification by regarding the data from $\widetilde{p}_+(\boldsymbol{x})$ as (noisy) positive data and the data from $\widetilde{p}_-(\boldsymbol{x})$ as (noisy) negative data. However, this naive solution could be inevitably affected by noisy labels. In this scenario, we denote the noise rates as $\rho_- = p(\widetilde{y}=+1|y=-1)$ and $\rho_+ = p(\widetilde{y}=-1|y=+1)$ where $\widetilde{y}$ is the observed (noisy) label and $y$ is the true label, and denote the inverse noise rates as $\phi_+ = p(y=-1|\widetilde{y}=+1)$ and $\phi_- = p(y=+1| \widetilde{y}=-1)$. According to the defined generation process of pairwise comparison data, we have the following theorem.
\begin{theo}
\label{noise_rates}
The following equalities hold:
\begin{align}
\nonumber
\phi_+ = \frac{\pi_-^2}{\pi_+^2+\pi_-^2+\pi_+\pi_-}&,\quad
\rho_+ =\frac{\pi_+}{1+\pi_+},\\
\nonumber
\phi_- = \frac{\pi_+^2}{\pi_+^2+\pi_-^2+\pi_+\pi_-}&,\quad
\rho_- = \frac{\pi_-}{1+\pi_-}.
\end{align}
\end{theo}
The proof is provided in Appendix F.

Theorem \ref{noise_rates} shows that the noise rates can be obtained if we regard the Pcomp classification problem as the noisy-label learning problem. With known noise rates for noisy-label learning, it was shown \cite{natarajan2013learning,northcutt2017learning} that a URE could be derived. Here, we adopt the RankPruning method \cite{northcutt2017learning} because it holds a progressive URE by selecting confident examples using the learning model and achieves state-of-the-art performance. Specifically, we denote by the dataset composed of all the observed positive data $\widetilde{\mathcal{P}}$, i.e., $\widetilde{\mathcal{P}}=\{\boldsymbol{x}_i\}_{i=1}^n$, where $\boldsymbol{x}_i$ is independently sampled from $\widetilde{p}_+(\boldsymbol{x})$. Similarly, the dataset composed of all the observed negative data is denoted by $\widetilde{\mathcal{N}}$, i.e., $\widetilde{\mathcal{N}}=\{\boldsymbol{x}^\prime_i\}_{i=1}^n$, where $\boldsymbol{x}^\prime_i$ is independently sampled from $\widetilde{p}_-(\boldsymbol{x}^\prime)$. Then, confident examples will be selected from $\widetilde{\mathcal{P}}$ and $\widetilde{\mathcal{N}}$ by ranking the outputs of the model $f$. We denote the selected positive data from $\widetilde{\mathcal{P}}$ as $\widetilde{\mathcal{P}}_{\mathrm{sel}}$, and the selected negative data from $\widetilde{\mathcal{N}}$ as $\widetilde{\mathcal{N}}_{\mathrm{sel}}$:
\begin{align}
\nonumber
\widetilde{\mathcal{P}}_{\mathrm{sel}} &= \argmax_{\mathcal{P}:\left|\mathcal{P}\right|=(1-\phi_+)\left|\widetilde{\mathcal{P}}\right|}\sum\nolimits_{\boldsymbol{x}\in\{\mathcal{P}\cap\widetilde{\mathcal{P}}\}}f(\boldsymbol{x}),\\
\nonumber
\widetilde{\mathcal{N}}_{\mathrm{sel}} &= \argmin_{\mathcal{N}:\left|\mathcal{N}\right|=(1-\phi_-)\left|\widetilde{\mathcal{N}}\right|}\sum\nolimits_{\boldsymbol{x}\in\{\mathcal{N}\cap\widetilde{\mathcal{N}}\}}f(\boldsymbol{x}).
\end{align}
Then we show that if the model $f$ satisfies the \emph{separability condition}, i.e., for any true positive instance $\boldsymbol{x}_{\mathrm{p}}$ and for any true negative instance $\boldsymbol{x}_{\mathrm{n}}$, we have $f(\boldsymbol{x}_{\mathrm{p}}) > f(\boldsymbol{x}_{\mathrm{n}})$. In other words, if the model output of every true positive instance is always larger than that of every true negative instance, we could obtain a URE. We name it progressive URE, as the model $f$ is progressively optimized.
\begin{theo}[Theorem 5 in \cite{northcutt2017learning}]
\label{noise_estimator}
Assume that the model $f$ satisfies the above separability condition, then the classification risk $R(f)$ can be equivalently expressed as
\begin{align}
\nonumber
R_{\mathrm{pPC}}(f) &= \mathbb{E}_{\widetilde{p}_+(\boldsymbol{x})}\Big[\frac{\ell(f(\boldsymbol{x}),+1)}{1-\rho_+}\mathbb{I}[\boldsymbol{x}\in\widetilde{\mathcal{P}}_{\mathrm{sel}}]\Big]\\
\nonumber
&\quad\quad\quad + \mathbb{E}_{\widetilde{p}_-(\boldsymbol{x}^\prime)}\Big[\frac{\ell(f(\boldsymbol{x}^\prime),-1)}{1-\rho_-}\mathbb{I}[\boldsymbol{x}^\prime\in\widetilde{\mathcal{N}}_{\mathrm{sel}}]\Big],
\end{align}
where $\mathbb{I}[\cdot]$ denotes the indicator function.
\end{theo}
In this way, we have the following empirical approximation of $R_{\mathrm{pPC}}$:
\begin{align}
\nonumber
\widehat{R}_{\mathrm{pPC}}(f) &= \frac{1}{n}\sum\nolimits_{i=1}^n\Big( \frac{\ell(f(\boldsymbol{x}_i),+1)}{1-\rho_+}\mathbb{I}[\boldsymbol{x}_i\in\widetilde{\mathcal{P}}_{\mathrm{sel}}]\\
\label{rankpruning}
&\quad\quad\quad\quad\quad + \frac{\ell(f(\boldsymbol{x}_i^\prime),-1)}{1-\rho_-}\mathbb{I}[\boldsymbol{x}_i^\prime\in\widetilde{\mathcal{N}}_{\mathrm{sel}}]\Big).
\end{align}
\noindent\textbf{Estimation Error Bound.}\quad
It worth noting that \citet{northcutt2017learning} did not prove the learning consistency for the RankPruning method. Here, we establish an estimation error bound for this method, which guarantees the learning consistency.
Let $\widehat{f}_{\mathrm{pPC}} = \argmin_{f\in\mathcal{F}}\widehat{R}_{\mathrm{pPC}}(f)$ be the empirical risk minimizer of the RankPruning method, then we have the following theorem.
\begin{theo}
\label{second_estimation_error}
Suppose the loss function $\ell$ is $\rho$-Lipschitz with respect to the first argument ($0\leq\rho\leq\infty$), and all functions in the model class $\mathcal{F}$ are bounded, i.e., there exists a positive constant $C_{\mathrm{b}}$ such that $\left\|f\right\|\leq C_{\mathrm{b}}$ for any $f\in\mathcal{F}$. Let $C_{\ell}:=\sup_{z\leq C_{\mathrm{b}}, t=\pm 1}\ell(z, t)$. Then for any $\delta>0$, with probability at least $1-\delta$, we have
\begin{align}
\nonumber
\textstyle
R(\widehat{f}_{\mathrm{pPC}})-R(f^\star)&\textstyle\leq \frac{2}{1-\rho_+}\Big(2\rho\widetilde{\mathfrak{R}}_n^+(\mathcal{F}) + C_{\ell}\sqrt{\frac{\log\frac{4}{\delta}}{2n}}\Big)\\
\nonumber
\textstyle
&\textstyle +\frac{2}{1-\rho_-}\Big(2\rho\widetilde{\mathfrak{R}}_n^-(\mathcal{F}) + C_{\ell}\sqrt{\frac{\log\frac{4}{\delta}}{2n}}\Big).
\end{align}
\end{theo}
The proof is provided in Appendix G. Theorem \ref{second_estimation_error} shows that the above method is consistent and this estimation error bound also attains the optimal convergence rate without any additional assumptions \cite{mendelson2008lower}.

\noindent\textbf{Regularization.}\quad
For the above RankPruning method, its URE is based on the assumption that the learning model could satisfy the separability condition. Thus, its performance heavily depends on the accuracy of the learning model. However, as the learning model is progressively updated, some of the selected confident examples may still contain label noise during the training process. As a result, the RankPruning method 
would be affected by incorrectly selected data. 
A straightforward improvement could be to improve the output quality of the learning model.
Motivated by Mean Teacher used in semi-supervised learning \cite{tarvainen2017mean},
we also resort to a teacher model that is an exponential moving average of model snapshots, i.e., $\boldsymbol{\Theta}_t^\prime = \alpha\boldsymbol{\Theta}_{t-1}^\prime + (1-\alpha)\boldsymbol{\Theta}_t$, where $\boldsymbol{\Theta}^\prime$ denotes the parameters of the teacher model, $\boldsymbol{\Theta}$ denotes the parameters of the learning model, the subscript $t$ denotes the training step, and $\alpha$ is a smoothing coefficient hyper-parameter. Such a teacher model could guide the learning model to produce high-quality outputs. To learn from the teacher model, we leverage consistency regularization $\Omega(f)=\mathbb{E}_{\boldsymbol{x}}\big[\|f_{\boldsymbol{\Theta}}(\boldsymbol{x}) - f_{\boldsymbol{\Theta}^\prime}(\boldsymbol{x})\|^2\big]$ \cite{laine2016temporal,tarvainen2017mean} to make the learning model consistent with the teacher model for improving the performance of the RankPruning method.

\section{Experiments}
In this section, we conduct extensive experiments to evaluate the performance of our proposed Pcomp classification methods on various datasets using different models.

\noindent\textbf{Datasets.}\quad
We use four popular benchmark datasets, including MNIST~\cite{lecun1998gradient}, Fashion-MNIST~\cite{xiao2017fashion}, Kuzushiji-MNIST~\cite{clanuwat2018deep}, and CIFAR-10~\cite{krizhevsky2009learning}. We train a multilayer perceptron (MLP) model with three hidden layers of width 300 and ReLU activation functions \cite{nair2010rectified} and batch normalization \cite{ioffe2015batch} on the first three datasets. We train a ResNet-34 model \cite{he2016deep} on the CIFAR-10 dataset. We also use USPS and three datasets from the UCI machine learning repository \cite{blake1998uci} including Pendigits, Optdigits, and CNAE-9. We train a linear model on these datasets, since they are not large-scale datasets. 
The brief descriptions of all used datasets with the corresponding models are reported in Table \ref{datasets}.
Since these datasets are specially used for multi-class classification, we manually transformed them into binary classification datasets (see Appendix H). As we have shown in Theorem \ref{relation}, the pairwise comparison examples can be equivalently transformed into pointwise examples, which are more convenient to generate. Therefore, we generate pointwise examples in experiments. Specifically, as Theorem \ref{noise_rates} discloses the noise rates in our defined data generation process, we simply generate pointwise corrupted examples according to the derived noise rates.

\noindent\textbf{Methods.}\quad 
For our proposed Pcomp classification problem, we propose the following methods:
\begin{itemize}[leftmargin=0.5cm]
\vspace{-10pt}
\item \textbf{Pcomp-Unbiased}, which denotes the proposed method that minimizes $\widehat{R}_{\mathrm{PC}}(f)$ in Eq.~(\ref{pc_empirical_estimator}).
\vspace{-5pt}
\item \textbf{Pcomp-ReLU}, which denotes the proposed method that minimizes $\widehat{R}_{\mathrm{cPC}}(f)$ in Eq.~(\ref{corrected_empirical_risk}) using the ReLU function as the risk correction function.
\vspace{-5pt}
\item \textbf{Pcomp-ABS}, which denotes the proposed method that minimizes $\widehat{R}_{\mathrm{cPC}}(f)$ in Eq.~(\ref{corrected_empirical_risk}) using the absolute value function as the risk correction function.
\vspace{-5pt}
\item \textbf{Pcomp-Teacher}, which improves the RankPruning method by imposing consistency regularization to make the learning model consistent with a teacher model.
\vspace{-10pt}
\end{itemize}
Besides, we compare with the following baselines:
\begin{itemize}[leftmargin=0.5cm]
\vspace{-10pt}
\item \textbf{Binary-Biased}, which conducts binary classification by regarding the data from $\widetilde{p}_+(\boldsymbol{x})$ as positive data and the data from $\widetilde{p}_-(\boldsymbol{x})$ as negative data. This is a straightforward method to handle the Pcomp classification problem. In our setting, Binary-Biased reduces to the BER minimization method \cite{menon2015learning}.
\vspace{-5pt}
\item \textbf{Noisy-Unbiased}, which denotes the noisy-label learning method that minimizes the empirical approximation of the URE proposed by \cite{natarajan2013learning}.
\vspace{-5pt}
\item \textbf{RankPruning}, which denotes the noisy-label learning method \cite{northcutt2017learning} that minimizes the empirical risk $\widehat{R}_{\mathrm{pPC}}(f)$ in Eq.~(\ref{rankpruning}).
\vspace{-10pt}
\end{itemize}
For all learning methods, we take the logistic loss as the binary loss function $\ell$ (i.e., $\ell(z)=\ln(1+\exp(-z))$), for fair comparisons. The hyper-parameter settings of all learning methods are aslo reported in Appendix H.
We implement our methods using PyTorch \cite{paszke2019pytorch} and use the
Adam \cite{kingma2015adam} optimizer with mini-batch size set to 256 and the number of training epochs set to 200 for the four large-scale datasets and 100 for other four datasets. All the experiments are conducted on GeForce GTX 1080 Ti GPUs. 

\begin{table}[!t]
\centering
\caption{Specification of the used benchmark datasets and models. These datasets are specially processed (see Appendix H) for performing Pcomp classification.}
\label{datasets}
\resizebox{0.485\textwidth}{!}{
\setlength{\tabcolsep}{1.0mm}{
				\begin{tabular}{lccccc}
					\toprule
					Dataset & \# Train & \# Test & \# Features & \# Classes & Model \\
					\midrule
					MNIST & 60,000 & 10,000 & 784 & 10 & MLP \\
					Fashion & 60,000 & 10,000 & 784 & 10 & MLP \\
					Kuzushiji & 60,000 & 10,000 & 784 & 10 & MLP \\
					CIFAR-10 & 50,000 & 10,000 & 3,072 & 10 & ResNet-34 \\
					\midrule
					USPS & 7,437 & 1,861 & 256 & 10 & Linear \\
					Pendigits & 8,793 & 2,199 & 16 & 10 & Linear \\
					Optdigits & 4,495 & 1,125 & 62 & 10 & Linear \\
					CNAE-9 & 864 & 216 & 856 & 9 & Linear  \\
					\bottomrule
				\end{tabular}
				}
				}
\vspace{-10pt}				
\end{table}

\begin{table*}[!t]
\centering
\caption{Classification accuracy (mean$\pm$std) of each method on the four benchmark datasets with different class priors. The best performance is highlighted in bold.}
\label{MLP_ResNet}
\resizebox{1.00\textwidth}{!}{
\setlength{\tabcolsep}{6.5mm}{
\begin{tabular}{c|c|c|c|c|c}
\toprule
Class Prior       & Methods & MNIST & Kuzushiji & Fashion & CIFAR-10 \\
\midrule
\multirow{7}{*}{$\pi_+=0.2$}	
    &Noisy-Unbiased	&0.806$\pm$0.026 	&0.712$\pm$0.014 	&0.865$\pm$0.050 	&0.665$\pm$0.119 	\\
	&Binary-Biased	&0.282$\pm$0.018 	&0.584$\pm$0.021 	&0.371$\pm$0.067 	&0.499$\pm$0.170 	\\
	&RankPruning	&0.933$\pm$0.005 	&0.810$\pm$0.006 	&0.938$\pm$0.005 	&0.840$\pm$0.005 	\\
	\cmidrule(r){2-6}
    &Pcomp-ABS	&0.893$\pm$0.013 	&\textbf{0.849$\pm$0.007} 	&0.898$\pm$0.008 	&0.833$\pm$0.008 	\\
	&Pcomp-ReLU	&0.927$\pm$0.010 	&0.838$\pm$0.014 	&0.927$\pm$0.022 	&0.812$\pm$0.007 	\\
	&Pcomp-Unbiased	&0.782$\pm$0.025 	&0.693$\pm$0.041 	&0.835$\pm$0.038 	&0.584$\pm$0.011 	\\
	&Pcomp-Teacher	&\textbf{0.958$\pm$0.007} 	&0.835$\pm$0.015 	&\textbf{0.954$\pm$0.004} 	&\textbf{0.841$\pm$0.006} 	\\
\midrule
\multirow{7}{*}{$\pi_+=0.5$}	&Noisy-Unbiased	&0.892$\pm$0.013 	&0.656$\pm$0.096 	&0.908$\pm$0.031 	&0.642$\pm$0.031 	\\
	&Binary-Biased	&0.537$\pm$0.026 	&0.615$\pm$0.008 	&0.457$\pm$0.032 	&0.470$\pm$0.039 	\\
	&RankPruning	&0.888$\pm$0.004 	&0.782$\pm$0.008 	&0.917$\pm$0.005 	&0.793$\pm$0.015 	\\
	\cmidrule(r){2-6}
	&Pcomp-ABS	&0.832$\pm$0.010 	&0.732$\pm$0.007 	&0.899$\pm$0.006 	&0.712$\pm$0.023 	\\
	&Pcomp-ReLU	&0.875$\pm$0.009 	&0.729$\pm$0.019 	&0.918$\pm$0.016 	&0.746$\pm$0.030 	\\
	&Pcomp-Unbiased	&0.877$\pm$0.012 	&0.690$\pm$0.084 	&0.906$\pm$0.037 	&0.631$\pm$0.031 	\\
	&Pcomp-Teacher	&\textbf{0.942$\pm$0.004} 	&\textbf{0.791$\pm$0.013} 	&\textbf{0.957$\pm$0.002} 	&\textbf{0.795$\pm$0.012} 	\\
\midrule
\multirow{7}{*}{$\pi_+=0.8$}	&Noisy-Unbiased	&0.800$\pm$0.036 	&0.749$\pm$0.038 	&0.815$\pm$0.058 	&0.591$\pm$0.019 	\\
	&Binary-Biased	&0.299$\pm$0.066 	&0.558$\pm$0.005 	&0.444$\pm$0.020 	&0.371$\pm$0.113 	\\
	&RankPruning	&0.939$\pm$0.004 	&0.830$\pm$0.012 	&0.939$\pm$0.004 	&0.843$\pm$0.010 	\\
	\cmidrule(r){2-6}
	&Pcomp-ABS	&0.815$\pm$0.009 	&0.825$\pm$0.011 	&0.879$\pm$0.014 	&0.827$\pm$0.013 	\\
	&Pcomp-ReLU	&0.916$\pm$0.013 	&0.827$\pm$0.011 	&0.925$\pm$0.015 	&0.811$\pm$0.007 	\\
	&Pcomp-Unbiased	&0.793$\pm$0.016 	&0.721$\pm$0.037 	&0.823$\pm$0.032 	&0.569$\pm$0.007 	\\
	&Pcomp-Teacher	&\textbf{0.958$\pm$0.007} 	&\textbf{0.836$\pm$0.015} 	&\textbf{0.955$\pm$0.004} 	&\textbf{0.844$\pm$0.014} 	\\
\bottomrule
\end{tabular}
}
}
\end{table*}

\noindent\textbf{Experimental Setup.}\quad 
We evaluate the performance of all learning methods under different class prior settings, i.e., $\pi_+$ is selected from $\{0.2, 0.5, 0.8\}$. It is noteworthy that we could empirically estimate $\pi_+$ if we exactly follow our defined data generation process to generate Pcomp examples. Specifically, we could first empirically estimate $\widetilde{\pi}=1-\pi_+\pi_-$ by the fraction of Pcomp examples from the three cases $\{(+1, -1), (+1, +1), (-1, -1)\}$ in all the sampled data pairs. Then $\pi_+$ can be estimated by $\pi_+=1/2-\sqrt{\widetilde{\pi}-3/4}$ (if $\pi_+<\pi_-$) or $\pi_+=1/2+\sqrt{\widetilde{\pi}-3/4}$ (if $\pi_+\geq \pi_-$). Therefore, if we know whether $\pi_+$ is larger than $\pi_-$, we could exactly estimate the true class prior $\pi_+$. For simplicity, we assume that the class prior $\pi_+$ is known for all the methods. We repeat the sampling-and-training process 5 times for all learning methods on all datasets and record the mean accuracy with standard deviation (mean$\pm$std).

\noindent\textbf{Experimental Results with Complex Models.}\quad Table \ref{MLP_ResNet} records the classification accuracy of each method on the four benchmark datasets with different class priors. From Table \ref{MLP_ResNet}, we have the following observations: 
\begin{itemize}[leftmargin=0.5cm]
\vspace{-10pt}
\item Binary-Biased always achieves the worst performance, which means that our Pcomp classification problem cannot be simply solved by binary classification.
\vspace{-5pt}
\item Pcomp-Unbiased is inferior to Pcomp-ABS and Pcomp-ReLU. This observation accords with what we have discussed, i.e., directly minimizing $\widehat{R}_{\mathrm{PC}}(f)$ would suffer from overfitting when complex models are used, because there are negative terms included in $\widehat{R}_{\mathrm{PC}}(f)$ and the empirical risk can be negative during the training process. In contrast, Pcomp-ReLU and Pcomp-ABS employ consistent correction functions on $\widehat{R}_{\mathrm{PC}}(f)$ so that the empirical risk will never be negative. Therefore, when complex models such as deep neural networks are used, Pcomp-ReLU and Pcomp-ABS are expected to outperform Pcomp-Unbiased.
\vspace{-5pt}
\item Pcomp-Teacher achieves the best performance in most cases. This observation verifies the effectiveness of the imposed consistency regularization, which makes the learning model consistent with a teacher model, to improve the quality of selected confident examples by the RankPruning method.
\vspace{-5pt}
\item It is worth noting that the standard deviations of Binary-Biased, Pcomp-Unbiased, and Noisy-Unbiased are sometimes higher than other methods. This is because the three methods suffer from overfitting when complex models are used, and the performance could be quite unstable in different trials.
\vspace{-5pt}
\end{itemize}
\begin{table*}[!t]
\centering
\caption{Classification accuracy (mean$\pm$std) of each method on the four UCI datasets with different class priors. The best performance is highlighted in bold.}
\label{linear}
\resizebox{1.00\textwidth}{!}{
\setlength{\tabcolsep}{6.5mm}{
\begin{tabular}{c|c|c|c|c|c}
\toprule
Class Prior       & Methods & USPS & Pendigits & Optdigits & CNAE-9 \\
\midrule
\multirow{7}{*}{$\pi_+=0.2$}   &Noisy-Unbiased	&0.921$\pm$0.010 	&0.857$\pm$0.012 	&0.885$\pm$0.015 	&0.830$\pm$0.019 	\\
	&Binary-Biased	&0.752$\pm$0.016 	&0.639$\pm$0.050 	&0.705$\pm$0.026 	&0.629$\pm$0.040 	\\
	&RankPruning	&0.931$\pm$0.007 	&0.780$\pm$0.061 	&0.878$\pm$0.011 	&0.780$\pm$0.046 	\\
                  \cmidrule(r){2-6}
    &Pcomp-ABS	&0.909$\pm$0.005 	&0.863$\pm$0.013 	&0.871$\pm$0.013 	&0.819$\pm$0.018 	\\
	&Pcomp-ReLU	&0.917$\pm$0.006 	&\textbf{0.868$\pm$0.012} 	&0.875$\pm$0.011 	&\textbf{0.835$\pm$0.011} 	\\
	&Pcomp-Unbiased	&0.922$\pm$0.007 	&0.867$\pm$0.015 	&0.881$\pm$0.018 	&0.806$\pm$0.023 	\\
	&Pcomp-Teacher	&\textbf{0.927$\pm$0.009} 	&0.847$\pm$0.035 	&\textbf{0.886$\pm$0.016} 	&0.687$\pm$0.052 	\\

\midrule
\multirow{7}{*}{$\pi_+=0.5$}  &Noisy-Unbiased	&0.911$\pm$0.007 	&0.826$\pm$0.014 	&0.838$\pm$0.011 	&0.766$\pm$0.039 	\\
	&Binary-Biased	&0.913$\pm$0.007 	&0.761$\pm$0.049 	&0.834$\pm$0.013 	&0.773$\pm$0.051 	\\
	&RankPruning	&0.925$\pm$0.008 	&0.840$\pm$0.019 	&0.864$\pm$0.027 	&0.672$\pm$0.068 	\\	
                  \cmidrule(r){2-6}
	&Pcomp-ABS	&0.912$\pm$0.008 	&0.839$\pm$0.010 	&0.841$\pm$0.013 	&0.749$\pm$0.047 	\\
	&Pcomp-ReLU	&0.912$\pm$0.007 	&0.844$\pm$0.008 	&0.846$\pm$0.013 	&0.766$\pm$0.038 	\\
	&Pcomp-Unbiased	&0.911$\pm$0.006 	&\textbf{0.846$\pm$0.009} 	&0.847$\pm$0.014 	&0.763$\pm$0.040 	\\
	&Pcomp-Teacher	&\textbf{0.926$\pm$0.008} 	&0.840$\pm$0.019 	&\textbf{0.869$\pm$0.021} 	&\textbf{0.787$\pm$0.047} 	\\
\midrule
\multirow{7}{*}{$\pi_+=0.8$}	&Noisy-Unbiased	&0.918$\pm$0.017 	&0.890$\pm$0.013 	&0.884$\pm$0.006 	&0.838$\pm$0.017 	\\
	&Binary-Biased	&0.731$\pm$0.017 	&0.634$\pm$0.042 	&0.554$\pm$0.106 	&0.621$\pm$0.037 	\\
	&RankPruning	&0.934$\pm$0.010 	&0.856$\pm$0.018 	&0.863$\pm$0.025 	&0.737$\pm$0.050 	\\
                  \cmidrule(r){2-6}
	&Pcomp-ABS	&0.900$\pm$0.017 	&0.889$\pm$0.014 	&0.882$\pm$0.002 	&0.824$\pm$0.013 	\\
	&Pcomp-ReLU	&0.921$\pm$0.018 	&\textbf{0.894$\pm$0.012} 	&\textbf{0.888$\pm$0.004} 	&0.839$\pm$0.015 	\\
	&Pcomp-Unbiased	&0.899$\pm$0.014 	&0.883$\pm$0.006 	&0.874$\pm$0.008 	&\textbf{0.840$\pm$0.013} 	\\
	&Pcomp-Teacher	&\textbf{0.937$\pm$0.008} 	&0.880$\pm$0.011 	&0.884$\pm$0.010 	&0.781$\pm$0.033 	\\
\bottomrule
\end{tabular}
}
}
\vspace{-5pt}
\end{table*}

\begin{figure*}[!t]
\centering
\subfigure{\includegraphics[width=1.6in]{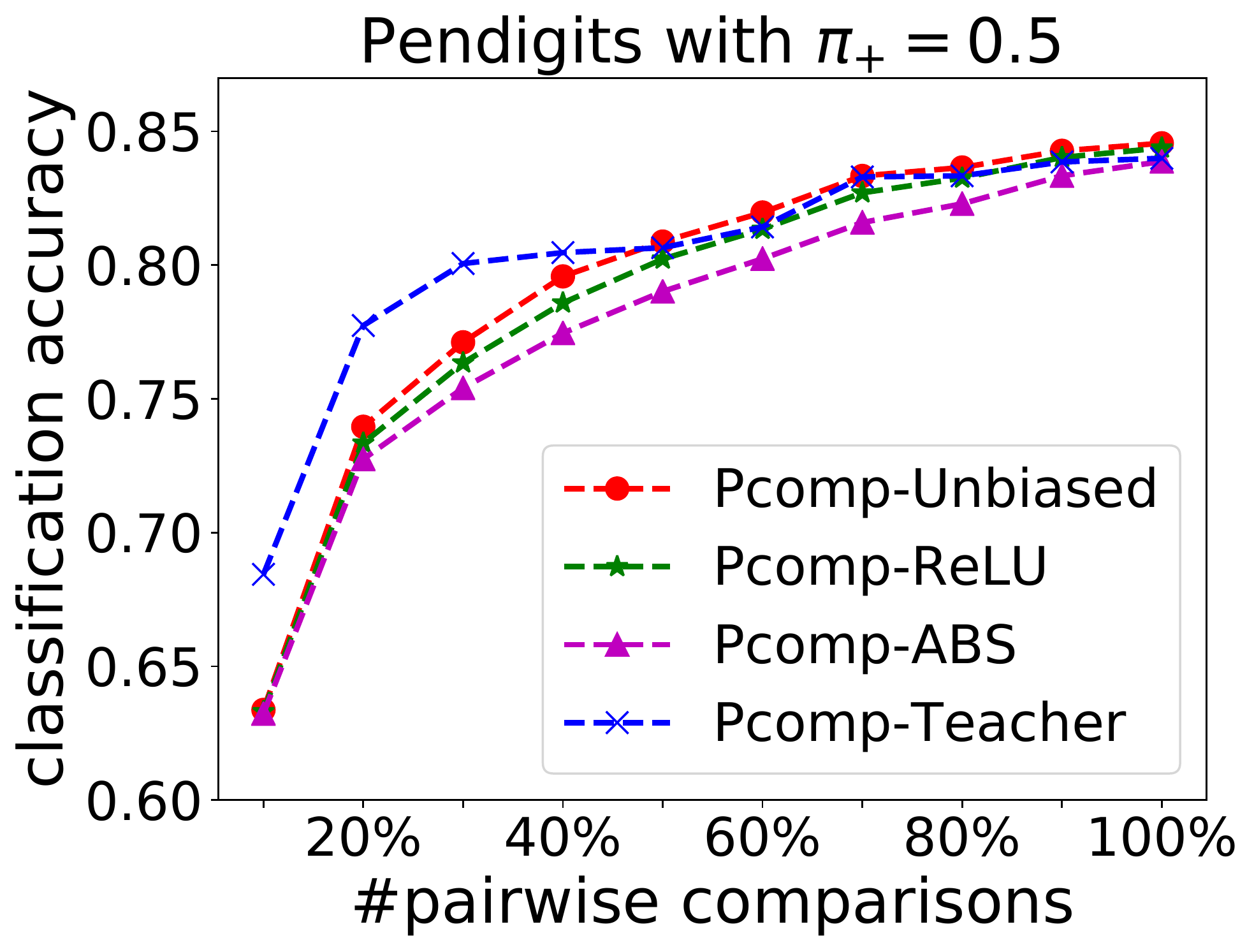}}
\subfigure{\includegraphics[width=1.6in]{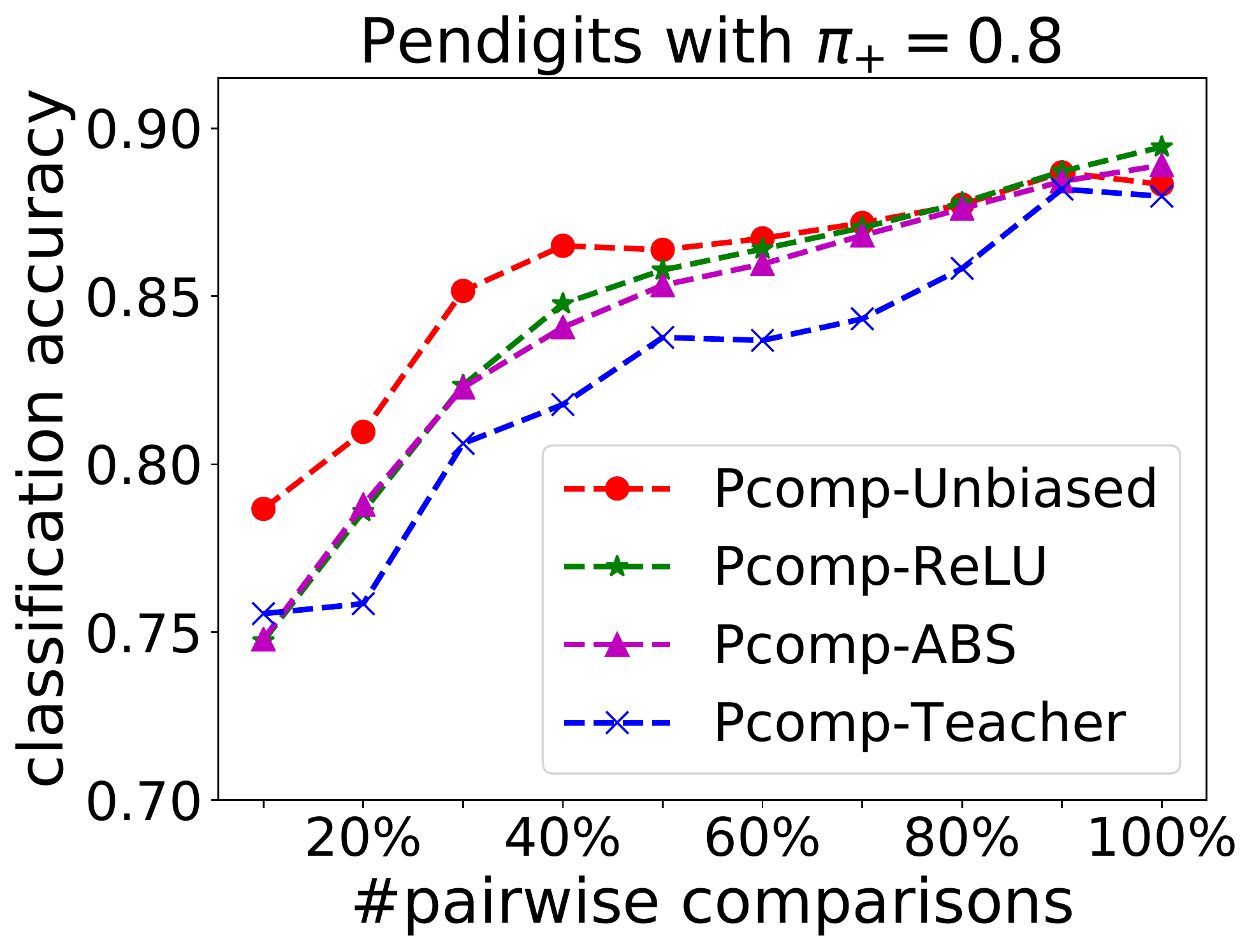}}
\subfigure{\includegraphics[width=1.6in]{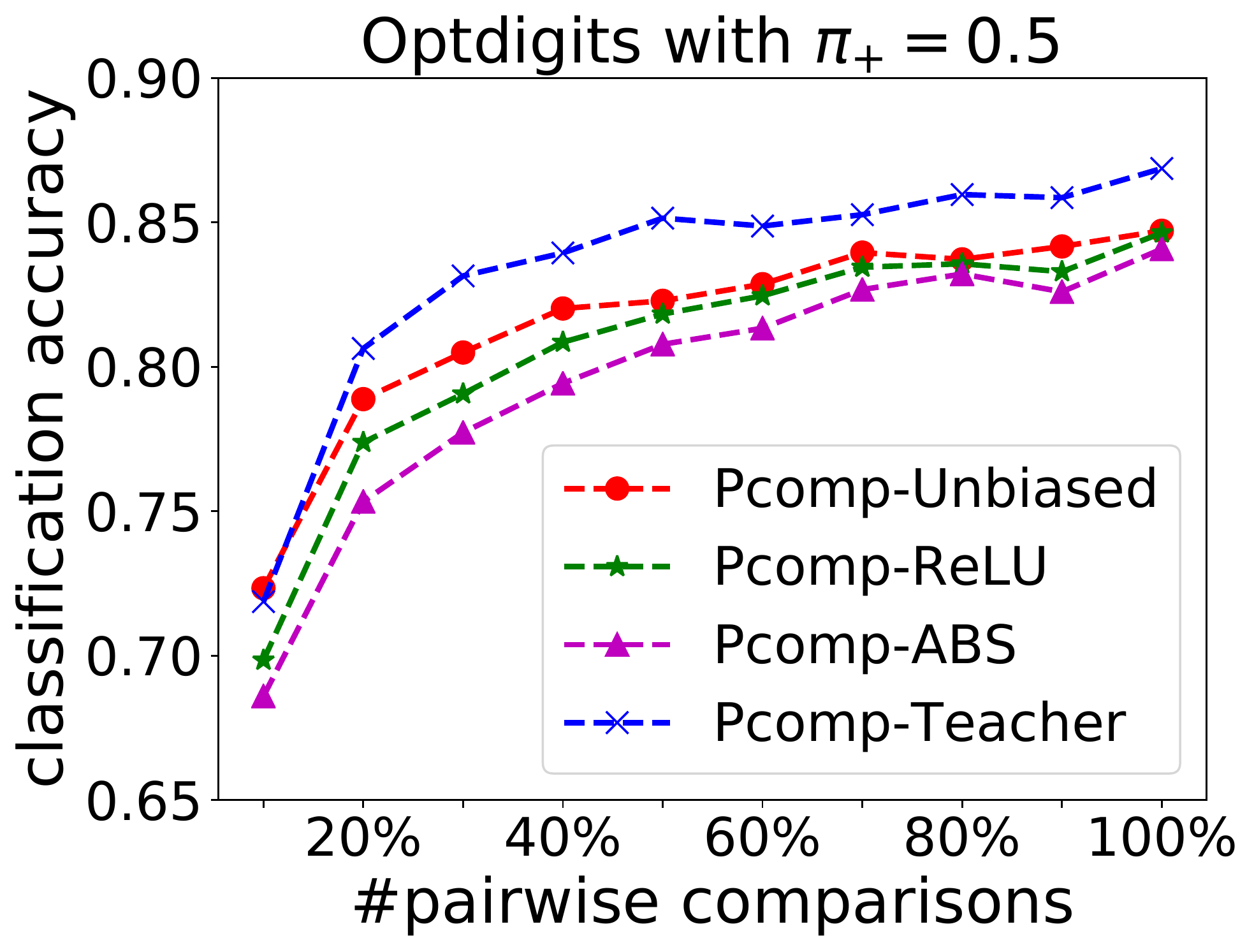}}
\subfigure{\includegraphics[width=1.6in]{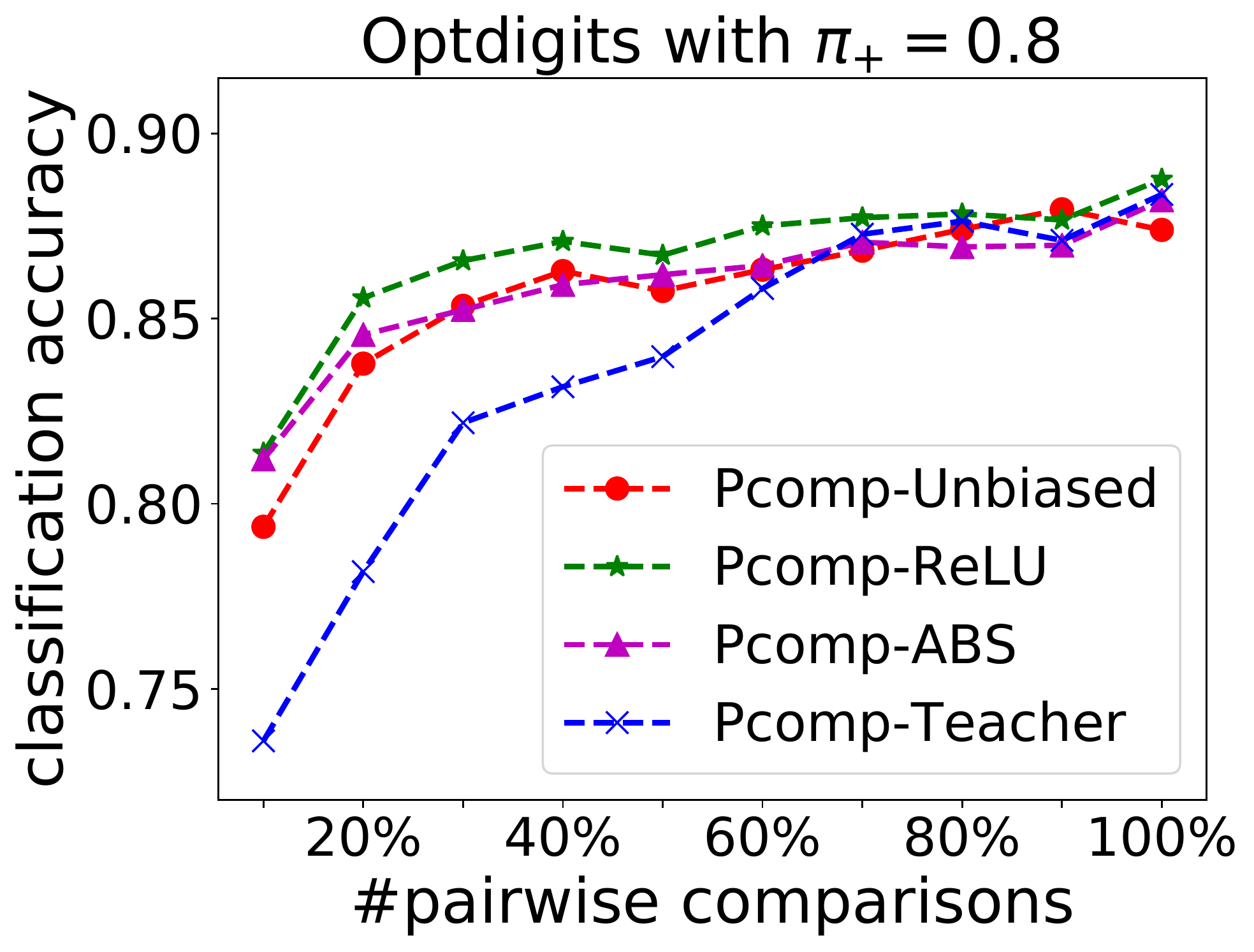}}
\vspace{-8pt}
\caption{The classification accuracy of our proposed Pcomp classification methods when the number of pairwise comparisons increases.}
\label{performance_increasing}
\vspace{-5pt}
\end{figure*}

\noindent\textbf{Experimental Results with Simple Models.}\quad Table \ref{linear} reports the classification accuracy of each method on the four UCI datasets with different class priors. From Table \ref{linear}, we have the following observations: 
\begin{itemize}[leftmargin=0.5cm]
\vspace{-10pt}
\item Binary-Biased achieves the worst performance in nearly all cases, which also implies that we need to develop other novel methods for Pcomp classification.
\vspace{-5pt}
\item Pcomp-Unbiased has comparable performance to its variants Pcomp-ABS and Pcomp-ReLU, because Pcomp-Unbiased does not suffer from overfitting when the linear model is used, and it is not necessary to use consistent correction functions anymore. 
\vspace{-5pt}
\item Pcomp-Teacher is still better than RankPruning, while it is sometimes inferior to other Pcomp classification methods. This is because the linear model is not as powerful as neural networks, thus the selected confident examples may not be so reliable.
\vspace{-5pt}
\end{itemize}

\noindent\textbf{Performance of Increasing Pairwise Comparisons.}\quad As shown by Theorem \ref{estimation_error} and Theorem \ref{second_estimation_error}, the performance of our Pcomp classification methods is expected to be improved if more pairwise comparisons are given.
To empirically validate such theoretical findings, we further conduct experiments on Pendigits and Optdigits with class prior $\pi_+=0.5$ and $\pi_+=0.8$ by changing the fraction of pairwise comparisons (100\% means that we use all the generated pairwise comparisons in the training process).
As shown in Figure \ref{performance_increasing}, the classification accuracy of our 
methods generally increases given more pairwise comparisons. This observation is clearly in accordance with our derived estimation error bounds, because the estimation error would decrease as the number of pairwise comparisons increases.

\section{Conclusion and Future Work}
In this paper, we proposed a novel weakly supervised learning setting called \emph{pairwise comparison (Pcomp) classification}, where we aim to train a binary classifier from only \emph{pairwise comparison data}, i.e., two examples that we know one is more likely to be positive than the other, instead of pointwise labeled data. Pcomp classification is useful for private classification tasks where we are not allowed to directly access labels and subjective classification tasks where labelers have different labeling standards. To solve the Pcomp classification problem, we presented a mathematical formulation for the generation process of pairwise comparison data, based on which we explored two \emph{unbiased risk estimators} (UREs) to train a binary classifier by \emph{empirical risk minimization} and established the corresponding \emph{estimation error bounds}.
We first proved that a URE can be derived and improved it using correction functions.
Then, we started from the \emph{noisy-label learning} perspective to introduce a \emph{progressive} URE and improved it by imposing \emph{consistency regularization}.
Finally, experiments demonstrated the effectiveness of our proposed methods.

In future work, we will apply Pcomp classification to solve some challenging real-world problems like binary classification with class overlapping. 
In addition, we could also extend Pcomp classification to the multi-class classification setting by using the one-versus-all strategy. Suppose there are multiple classes, we are given pairs of unlabeled data that we know which one is more likely to belong to a specific class. Then, we can use the proposed methods in this paper to train a binary classifier for each class. Finally, by comparing the outputs of these binary classifiers, the predicted class can be determined. 
\section*{Acknowledgements}
This research was supported by the National Research Foundation, Singapore under its AI Singapore Programme (AISG Award No: AISG-RP-2019-0013), National Satellite of Excellence in Trustworthy Software Systems (Award No: NSOE-TSS2019-01), and NTU.
NL was supported by MEXT scholarship No.\ 171536 and the MSRA D-CORE Program. BH was supported by the RGC Early Career Scheme No. 22200720, NSFC Young Scientists Fund No. 62006202, and HKBU CSD Departmental Incentive Grant.
GN and MS were supported by JST AIP Acceleration Research Grant Number JPMJCR20U3, Japan. MS was also supported by the Institute for AI and Beyond, UTokyo.
\bibliographystyle{icml2021}
\bibliography{example_paper}

\newpage
\onecolumn
\appendix
\section{Proof of Theorem \ref{data_generation}}\label{proof_of_theorem1}
Let us define $\widetilde{\mathcal{Y}}=\{(+1,+1),(+1,-1),(-1,-1)\}$. It is clear that each pair of examples $(\boldsymbol{x},\boldsymbol{x}^\prime)$ is independently drawn from the following data distribution:
\begin{align}
\nonumber
\widetilde{p}(\boldsymbol{x},\boldsymbol{x}^\prime) &=p((\boldsymbol{x},\boldsymbol{x}^\prime)\mid (y,y^\prime)\in\widetilde{\mathcal{Y}})= \frac{p((\boldsymbol{x},\boldsymbol{x}^\prime), (y,y^\prime)\in\widetilde{\mathcal{Y}})}{p((y,y^\prime)\in\widetilde{\mathcal{Y}})},
\end{align}
where $p((y,y^\prime)\in\widetilde{\mathcal{Y}}) = \pi_+^2 + \pi_-^2 + \pi_+\pi_-$ and
\begin{align}
\nonumber
p(\boldsymbol{x},\boldsymbol{x}^\prime, (y,y^\prime)\in\widetilde{\mathcal{Y}})&= \sum\nolimits_{(y,y^\prime)\in\widetilde{\mathcal{Y}}}p(\boldsymbol{x},\boldsymbol{x}^\prime\mid (y,y^\prime))\cdot p(y,y^\prime)\\
\nonumber
&= \pi_+^2p_+(\boldsymbol{x})p_+(\boldsymbol{x}^\prime) + \pi_-^2p_-(\boldsymbol{x})p_-(\boldsymbol{x}^\prime) + \pi_+\pi_-p_+(\boldsymbol{x})p_-(\boldsymbol{x}^\prime).
\end{align}
Finally, let $\widetilde{p}(\boldsymbol{x},\boldsymbol{x}^\prime) =p((\boldsymbol{x},\boldsymbol{x}^\prime)\mid (y,y^\prime)\in\widetilde{\mathcal{Y}})$, the proof is completed.\qed
\section{Proof of Theorem \ref{relation}}\label{proof_of_theorem2}
In order to decompose the pairwise comparison data distribution into pointwise distribution, we marginalize $\widetilde{p}(\boldsymbol{x},\boldsymbol{x}^\prime)$ with respect to $\boldsymbol{x}$ or $\boldsymbol{x}^\prime$. Then we can obtain
\begin{align}
\nonumber
\int \widetilde{p}(\boldsymbol{x},\boldsymbol{x}^\prime)\mathrm{d}\boldsymbol{x}^\prime
=& \frac{1}{\widetilde{\pi}}\Big( \pi_+^2p_+(\boldsymbol{x})+\pi_-^2p_-(\boldsymbol{x})+\pi_+\pi_-p_+(\boldsymbol{x})\Big)\\
\nonumber
=& \frac{\pi_+}{\pi_-^2+\pi_+}p_+(\boldsymbol{x}) + \frac{\pi_-^2}{\pi_-^2+\pi_+}p_-(\boldsymbol{x})\\
\nonumber
=& \widetilde{p}_+(\boldsymbol{x}),
\end{align}
and
\begin{align}
\nonumber
\int \widetilde{p}(\boldsymbol{x},\boldsymbol{x}^\prime)\mathrm{d}\boldsymbol{x}
=& \frac{1}{\widetilde{\pi}}\big( \pi_+^2p_+(\boldsymbol{x}^\prime)+\pi_-^2p_-(\boldsymbol{x}^\prime)+\pi_+\pi_-p_-(\boldsymbol{x}^\prime)\big)\\
\nonumber
=& \frac{\pi_+^2}{\pi_+^2+\pi_-}p_+(\boldsymbol{x}^\prime) + \frac{\pi_-}{\pi_+^2+\pi_-}p_-(\boldsymbol{x}^\prime)\\
\nonumber
=& \widetilde{p}_-(\boldsymbol{x}^\prime),
\end{align}
which concludes the proof of Theorem \ref{relation}.\qed
\section{Proof of Lemma \ref{another_relation}}\label{proof_of_lemma1}
Based on Theorem \ref{relation}, we can obtain the following linear equation:
\begin{gather}
\nonumber
\begin{bmatrix}
\widetilde{p}_+(\boldsymbol{x})\\
\widetilde{p}_-(\boldsymbol{x})
\end{bmatrix} = \frac{1}{\widetilde{\pi}}\begin{bmatrix}
\pi_+ & \pi_-^2\\
\pi_+^2 & \pi_-
\end{bmatrix}
\begin{bmatrix}
p_+(\boldsymbol{x})\\
p_-(\boldsymbol{x})
\end{bmatrix}.
\end{gather}
By solving the above equation, we obtain
\begin{align}
\nonumber
p_+(\boldsymbol{x}) &= \frac{1}{\pi_+-\pi_-\pi_+^2}\big(\widetilde{\pi}\cdot\widetilde{p}_+(\boldsymbol{x}) - \pi_-\widetilde{\pi}\cdot\widetilde{p}_-(\boldsymbol{x})\big)= \frac{1}{\pi_+}\big(\widetilde{p}_+(\boldsymbol{x}) - \pi_-\widetilde{p}_-(\boldsymbol{x})\big),\\
\nonumber
p_-(\boldsymbol{x}) &= \frac{1}{\pi_- - \pi_+\pi_-^2}\big(\widetilde{\pi}\cdot\widetilde{p}_-(\boldsymbol{x}) - \pi_+\widetilde{\pi}\cdot\widetilde{p}_+(\boldsymbol{x})\big)= \frac{1}{\pi_-}\big(\widetilde{p}_-(\boldsymbol{x}) - \pi_+\widetilde{p}_+(\boldsymbol{x})\big),
\end{align}
which concludes the proof of Lemma \ref{another_relation}.\qed
\section{Proof of Theorem \ref{pc_estimator}}\label{proof_of_theorem3}
It is quite intuitive to derive
\begin{align}
\nonumber
R(f)&= \mathbb{E}_{p(\boldsymbol{x},y)}\big[\ell(f(\boldsymbol{x}),y)\big]\\
\nonumber
&= \pi_+\mathbb{E}_{p_+(\boldsymbol{x})}\big[\ell(f(\boldsymbol{x}),+1)\big] + \pi_-\mathbb{E}_{p_-(\boldsymbol{x})}\big[\ell(f(\boldsymbol{x}),-1)\big]\\
\nonumber
&= \frac{\pi_+\widetilde{\pi}}{\pi_+-\pi_-\pi_+^2}\mathbb{E}_{\widetilde{p}_+(\boldsymbol{x})}\big[\ell(f(\boldsymbol{x}),+1)\big]
- \frac{\pi_+\pi_-\widetilde{\pi}}{\pi_+-\pi_-\pi_+^2}\mathbb{E}_{\widetilde{p}_-(\boldsymbol{x}^\prime)}\big[\ell(f(\boldsymbol{x}),+1)\big]\quad\quad  (\text{Lemma \ref{another_relation}})\\
\nonumber
&\quad\quad\quad\quad\quad\quad\quad+\frac{\pi_-\widetilde{\pi}}{\pi_--\pi_+\pi_-^2}\mathbb{E}_{\widetilde{p}_-(\boldsymbol{x}^\prime)}\big[\ell(f(\boldsymbol{x}),-1)\big]
- \frac{\pi_+\pi_-\widetilde{\pi}}{\pi_--\pi_+\pi_-^2}\mathbb{E}_{\widetilde{p}_+(\boldsymbol{x})}\big[\ell(f(\boldsymbol{x}),-1)\big]\\
\nonumber
&= \mathbb{E}_{\widetilde{p}_+(\boldsymbol{x})}\big[\ell(f(\boldsymbol{x}),+1)-\pi_+\ell(f(\boldsymbol{x}),-1)\big]
+ \mathbb{E}_{\widetilde{p}_-(\boldsymbol{x}^\prime)}\big[\ell(f(\boldsymbol{x}),-1)-\pi_-\ell(f(\boldsymbol{x}),+1)\big]\\
\nonumber
&= R_{\mathrm{PC}}(f),
\end{align}
which concludes the proof of Theorem \ref{pc_estimator}.\qed
\section{Proof of Theorem \ref{estimation_error}}\label{proof_of_theorem4}
First of all, we introduce the following notations:
\begin{align}
\nonumber
R^+_{\mathrm{PC}}(f) &= \mathbb{E}_{\widetilde{p}_+(\boldsymbol{x})}\big[\ell(f(\boldsymbol{x}),+1)-\pi_+\ell(f(\boldsymbol{x}),-1)\big],\\
\nonumber
\widehat{R}^+_{\mathrm{PC}}(f) &= \frac{1}{n}\sum_{i=1}^n\Big(\ell(f(\boldsymbol{x}_i),+1)-\pi_+\ell(f(\boldsymbol{x}_i),-1)\Big),\\
\nonumber
R^-_{\mathrm{PC}}(f) &= \mathbb{E}_{\widetilde{p}_-(\boldsymbol{x}^\prime)}\big[\ell(f(\boldsymbol{x}^\prime),-1)-\pi_-\ell(f(\boldsymbol{x}^\prime),+1)\big],\\
\nonumber
\widehat{R}^-_{\mathrm{PC}}(f) &= \frac{1}{n}\sum_{i=1}^n\Big(\ell(f(\boldsymbol{x}_i^\prime),-1)-\pi_-\ell(f(\boldsymbol{x}_i^\prime),+1)\Big).
\end{align}
In this way, we could simply represent $R_{\mathrm{PC}}(f)$ and $\widehat{R}_{\mathrm{PC}}(f)$ as
\begin{gather}
\nonumber
R_{\mathrm{PC}}(f) = R^+_{\mathrm{PC}}(f) + R^-_{\mathrm{PC}}(f),\quad
\widehat{R}_{\mathrm{PC}}(f) = \widehat{R}^+_{\mathrm{PC}}(f) + \widehat{R}^-_{\mathrm{PC}}(f).
\end{gather}
Then we have the following lemma.
\begin{lemma}
\label{inequality_lemma}
The following inequality holds:
\begin{gather}
\label{first_inequality}
R(\widehat{f}_{\mathrm{PC}}) - R(f^\star)\leq 2\sup_{f\in\mathcal{F}}\left|R_{\mathrm{PC}}^+(f) - \widehat{R}^+_{\mathrm{PC}}(f)\right| + 2\sup_{f\in\mathcal{F}}\left|R_{\mathrm{PC}}^-(f) - \widehat{R}_{\mathrm{PC}}^-(f)\right|.
\end{gather}
\end{lemma}
\begin{proof}
We could intuitively express $R(\widehat{f}_{\mathrm{PC}}) - R(f^\star)$ as
\begin{align}
\nonumber
R(\widehat{f}_{\mathrm{PC}}) - R(f^\star) &= R(\widehat{f}_{\mathrm{PC}}) - \widehat{R}_{\mathrm{PC}}(\widehat{f}_{\mathrm{PC}}) + \widehat{R}_{\mathrm{PC}}(\widehat{f}_{\mathrm{PC}}) - \widehat{R}_{\mathrm{PC}}(f^\star) + \widehat{R}_{\mathrm{PC}}(f^\star) - R(f^\star)\\
\nonumber
&= R_{\mathrm{PC}}(\widehat{f}_{\mathrm{PC}}) - \widehat{R}_{\mathrm{PC}}(\widehat{f}_{\mathrm{PC}}) + \widehat{R}_{\mathrm{PC}}(\widehat{f}_{\mathrm{PC}}) - \widehat{R}_{\mathrm{PC}}(f^\star) + \widehat{R}_{\mathrm{PC}}(f^\star) - R_{\mathrm{PC}}(f^\star)\\
\nonumber
&\leq \sup_{f\in\mathcal{F}}\left|R_{\mathrm{PC}}(f) - \widehat{R}_{\mathrm{PC}}(f)\right| + 0 + \sup_{f\in\mathcal{F}}\left|R_{\mathrm{PC}}(f) - \widehat{R}_{\mathrm{PC}}(f)\right|\\
\nonumber
&=2\sup_{f\in\mathcal{F}}\left|R_{\mathrm{PC}}(f) - \widehat{R}_{\mathrm{PC}}(f)\right|\\
\nonumber
&\leq 2\sup_{f\in\mathcal{F}}\left|R_{\mathrm{PC}}^+(f) - \widehat{R}^+_{\mathrm{PC}}(f)\right| + 2\sup_{f\in\mathcal{F}}\left|R_{\mathrm{PC}}^-(f) - \widehat{R}_{\mathrm{PC}}^-(f)\right|,
\end{align}
where the second inequality holds due to Theorem \ref{pc_estimator}.
\end{proof}
As suggested by Lemma \ref{inequality_lemma}, we need to further upper bound the right hand size of Eq.~(\ref{first_inequality}). Before doing that, we introduce the \emph{uniform deviation bound}, which is useful to derive estimation error bounds. The proof can be found in some textbooks such as \cite{mohri2012foundations} (Theorem 3.1).
\begin{lemma}
\label{uniform_bound}
Let $Z$ be a random variable drawn from a probability distribution with density $\mu$, $\mathcal{H}=\{h:\mathcal{Z}\mapsto [0,M]\}$ ($M>0$) be a class of measurable functions, $\{z_i\}_{i=1}^n$ be i.i.d. examples drawn from the distribution with density $\mu$. Then, for any $delta>0$, with probability at least $1-\delta$,
\begin{gather}
\nonumber
\sup_{h\in\mathcal{H}}\left|\mathbb{E}_{Z\sim\mu}\big[h(Z)\big] - \frac{1}{n}\sum_{i=1}^nh(z_i)\right|\leq 2\mathfrak{R}_n(\mathcal{H}) + M\sqrt{\frac{\log\frac{2}{\delta}}{2n}},
\end{gather}
where $\mathfrak{R}_n(\mathcal{H})$ denotes the (expected) \emph{Rademacher complexity} \cite{bartlett2002rademacher} of $\mathcal{H}$ with sample size $n$ over $\mu$.
\end{lemma}
\begin{lemma}
\label{lemma_of_first_inequality}
Suppose the loss function $\ell$ is $\rho$-Lipschitz with respect to the first argument ($0<\rho<\infty$), and all the functions in the model class $\mathcal{F}$ are bounded, i.e., there exists a constant $C_{\mathrm{b}}$ such that $\left\|f\right\|_{\infty}\leq C_{\mathrm{b}}$ for any $f\in\mathcal{F}$. Let $C_{\ell}:=\sup_{t=\pm 1}\ell(C_{\mathrm{b}},t)$. For any $\delta>0$, with probability $1-\delta$,
\begin{gather}
\nonumber
\sup_{f\in\mathcal{F}}\left|R_{\mathrm{PC}}^+(f) - \widehat{R}^+_{\mathrm{PC}}(f)\right|\leq (1+\pi_+)2\rho\widetilde{\mathfrak{R}}_n^+(\mathcal{F}) + (1+\pi_+)C_{\ell}\sqrt{\frac{\log\frac{4}{\delta}}{2n}}.
\end{gather}
\end{lemma}
\begin{proof}
By the definition of $R_{\mathrm{PC}}^+(f)$ and $\widehat{R}^+_{\mathrm{PC}}(f)$, we can obtain
\begin{align}
\nonumber
\sup_{f\in\mathcal{F}}\left|R_{\mathrm{PC}}^+(f) - \widehat{R}^+_{\mathrm{PC}}(f)\right| &\leq \sup_{f\in\mathcal{F}}\left|\mathbb{E}_{\widetilde{p}_+(\boldsymbol{x})}\big[\ell(f(\boldsymbol{x}),+1)\big] - \frac{1}{n}\sum_{i=1}^n\ell(f(\boldsymbol{x}),+1)
\right|\\
\label{positive_eq1}
&\quad\quad\quad\quad+\pi_+\sup_{f\in\mathcal{F}}\left|\mathbb{E}_{\widetilde{p}_+(\boldsymbol{x})}\big[\ell(f(\boldsymbol{x}),-1)\big] - \frac{1}{n}\sum_{i=1}^n\ell(f(\boldsymbol{x}),-1)
\right|.
\end{align}
By applying Lemma \ref{uniform_bound}, we have for any $\delta>0$, with probability $1-\delta$,
\begin{gather}
\label{positive_eq2}
\sup_{f\in\mathcal{F}}\left|\mathbb{E}_{\widetilde{p}_+(\boldsymbol{x})}\big[\ell(f(\boldsymbol{x}),+1)\big] - \frac{1}{n}\sum_{i=1}^n\ell(f(\boldsymbol{x}),+1)
\right|\leq 2\widetilde{\mathfrak{R}}_n^+(\ell\circ\mathcal{F})+C_{\ell}\sqrt{\frac{\log\frac{2}{\delta}}{2n}},
\end{gather}
and for any for any $\delta>0$, with probability $1-\delta$,
\begin{gather}
\label{positive_eq3}
\sup_{f\in\mathcal{F}}\left|\mathbb{E}_{\widetilde{p}_+(\boldsymbol{x})}\big[\ell(f(\boldsymbol{x}),-1)\big] - \frac{1}{n}\sum_{i=1}^n\ell(f(\boldsymbol{x}),-1)
\right|\leq 2\widetilde{\mathfrak{R}}_n^+(\ell\circ\mathcal{F})+C_{\ell}\sqrt{\frac{\log\frac{2}{\delta}}{2n}},
\end{gather}
where $\ell\circ\mathcal{F}$ means $\{\ell\circ f\mid f\in\mathcal{F}\}$. By Talagrand's lemma (Lemma 4.2 in \cite{mohri2012foundations}),
\begin{gather}
\label{talagrand}
\widetilde{\mathfrak{R}}_n^+(\ell\circ\mathcal{F})\leq \rho\widetilde{\mathfrak{R}}_n^+(\mathcal{F}).
\end{gather}
Finally, by combing Eqs. (\ref{positive_eq1}), (\ref{positive_eq2}), (\ref{positive_eq3}), and (\ref{talagrand}), we have for any $\delta>0$, with probability at least $1-\delta$,
\begin{gather}
\sup_{f\in\mathcal{F}}\left|R_{\mathrm{PC}}^+(f) - \widehat{R}^+_{\mathrm{PC}}(f)\right|\leq (1+\pi_+)2\rho\widetilde{\mathfrak{R}}_n^+(\mathcal{F}) + (1+\pi_+)C_{\ell}\sqrt{\frac{\log\frac{4}{\delta}}{2n}},
\end{gather}
which concludes the proof of Lemma \ref{lemma_of_first_inequality}.
\end{proof}
\begin{lemma}
\label{lemma_of_second_inequality}
Suppose the loss function $\ell$ is $\rho$-Lipschitz with respect to the first argument ($0<\rho<\infty$), and all the functions in the model class $\mathcal{F}$ are bounded, i.e., there exists a constant $C_{\mathrm{b}}$ such that $\left\|f\right\|_{\infty}\leq C_{\mathrm{b}}$ for any $f\in\mathcal{F}$. Let $C_{\ell}:=\sup_{t=\pm 1}\ell(C_{\mathrm{b}},t)$. For any $\delta>0$, with probability $1-\delta$,
\begin{gather}
\nonumber
\sup_{f\in\mathcal{F}}\left|R_{\mathrm{PC}}^-(f) - \widehat{R}^-_{\mathrm{PC}}(f)\right|\leq (1+\pi_-)2\rho\widetilde{\mathfrak{R}}_n^-(\mathcal{F}) + (1+\pi_-)C_{\ell}\sqrt{\frac{\log\frac{4}{\delta}}{2n}}.
\end{gather}
\end{lemma}
\begin{proof}
Lemma \ref{lemma_of_second_inequality} can be proved similarly to Lemma \ref{lemma_of_first_inequality}.
\end{proof}
By combining Lemma \ref{inequality_lemma}, Lemma \ref{lemma_of_first_inequality}, and Lemma \ref{lemma_of_second_inequality}, Theorem \ref{estimation_error} is proved.\qed
\section{Proof of Theorem \ref{noise_rates}}\label{proof_of_theorem5}
Suppose there are $n$ pairs of paired data points, which means there are in total $2n$ data points. For our Pcomp classification problem, we could simply regard $\boldsymbol{x}$ sampled from $\widetilde{p}_+(\boldsymbol{x})$ as (noisy) positive data and $\boldsymbol{x}^\prime$ sampled from $\widetilde{p}_-(\boldsymbol{x}^\prime)$ as (noisy) negative data. 
Given $n$ pairs of examples $\{(\boldsymbol{x}_i,\boldsymbol{x}_i^\prime)\}_{i=1}^n$, for the $n$ observed positive examples, there are actually $n\cdot p(y=+1| \widetilde{y}=+1)$ true positive examples; for the $n$ observed negative examples, there are actually $n\cdot p(y=-1| \widetilde{y}=-1)$ true negative examples. From our defined data generation process in Theorem \ref{data_generation}, it is intuitive to obtain
\begin{align}
\nonumber
p(y=+1\mid\widetilde{y}=+1) &= \frac{\pi_+^2+\pi_+\pi_-}{\pi_+^2+\pi_-^2+\pi_+\pi_-} = \frac{\pi_+}{\pi_+^2+\pi_-^2+\pi_+\pi_-},\\
\nonumber
p(y=-1\mid\widetilde{y}=-1) &= \frac{\pi_-^2+\pi_+\pi_-}{\pi_+^2+\pi_-^2+\pi_+\pi_-} = \frac{\pi_-}{\pi_+^2+\pi_-^2+\pi_+\pi_-}.
\end{align}
Since $\phi_{+} = p(y=-1\mid\widetilde{y}=+1)=1-p(y=+1\mid \widetilde{y}=+1)$ and $\phi_{-} = p(y=+1\mid \widetilde{y}=-1)=1-p(y=-1\mid\widetilde{y}=-1)$, we can obtain
\begin{align}
\nonumber
\phi_{+} = p(y=-1\mid\widetilde{y}=+1) &= 1 - \frac{\pi_+}{\pi_+^2+\pi_-^2+\pi_+\pi_-} = \frac{\pi_-^2}{\pi_+^2+\pi_-^2+\pi_+\pi_-},\\
\nonumber
\phi_{-} = p(y=+1\mid\widetilde{y}=-1) &= 1 - \frac{\pi_-}{\pi_+^2+\pi_-^2+\pi_+\pi_-}
= \frac{\pi_+^2}{\pi_+^2+\pi_-^2+\pi_+\pi_-}.
\end{align}
In this way, we can further obtain the following noise transition ratios:
\begin{align}
\nonumber
\rho_{+} &= p(\widetilde{y}=-1\mid y=+1) = \frac{p(y=+1\mid\widetilde{y}=-1)p(\widetilde{y}=-1)}{p(y=+1\mid\widetilde{y}=-1)p(\widetilde{y}=-1)+p(y=+1\mid\widetilde{y}=+1)p(\widetilde{y}=+1)}=\frac{\pi_+}{1+\pi_+}
,\\
\nonumber
\rho_{-} &= p(\widetilde{y}=+1\mid y=-1) = \frac{p(y=-1\mid\widetilde{y}=+1)p(\widetilde{y}=+1)}{p(y=-1\mid\widetilde{y}=+1)p(\widetilde{y}=+1)+p(y=-1\mid\widetilde{y}=-1)p(\widetilde{y}=-1)}= \frac{\pi_-}{1+\pi_-},
\end{align}
where $p(\widetilde{y}=1) = p(\widetilde{y}=-1) = \frac{1}{2}$, because we have the same number of observed positive examples and negative examples.
\section{Proof of Theorem \ref{second_estimation_error}}\label{proof_of_theorem7}
First of all, we introduce the following notations:
\begin{align}
\nonumber
R_{\mathrm{pPC}}^+(f) &= \mathbb{E}_{\widetilde{p}_+(\boldsymbol{x})}\big[\ell(f(\boldsymbol{x}),+1)\mathbb{I}[\boldsymbol{x}\in\mathrm{P}\widetilde{\mathrm{P}}]\big],\\
\nonumber
\widehat{R}_{\mathrm{pPC}}^+(f) &= 
\frac{1}{n}\sum_{i=1}^n\big(\ell(f(\boldsymbol{x}_i),+1)\mathbb{I}[\boldsymbol{x}_i\in\mathrm{P}\widetilde{\mathrm{P}}]\big),\\
\nonumber
R_{\mathrm{pPC}}^-(f) &=
\mathbb{E}_{\widetilde{p}_-(\boldsymbol{x}^\prime)}\big[\ell(f(\boldsymbol{x}^\prime),-1)\mathbb{I}[\boldsymbol{x}^\prime\in\mathrm{N}\widetilde{\mathrm{N}}]\big],\\
\nonumber
\widehat{R}_{\mathrm{pPC}}^-(f)&=
\frac{1}{n}\sum_{i=1}^n\big(\ell(f(\boldsymbol{x}_i^\prime),-1)\mathbb{I}[\boldsymbol{x}_i^\prime\in\mathrm{N}\widetilde{\mathrm{N}}]\big).
\end{align}
In this way, we could simply represent $R_{\mathrm{ppc}}(f)$ and $\widehat{R}_{\mathrm{pPC}}(f)$ as
\begin{gather}
\nonumber
R_{\mathrm{pPC}}(f) = \frac{1}{1-\rho_+}R^+_{\mathrm{pPC}}(f) + \frac{1}{1-\rho_-}R^-_{\mathrm{pPC}}(f),\quad
\widehat{R}_{\mathrm{pPC}}(f) = \frac{1}{1-\rho_+}\widehat{R}^+_{\mathrm{pPC}}(f) + \frac{1}{1-\rho_-}\widehat{R}^-_{\mathrm{pPC}}(f).
\end{gather}
Then we have the following lemma.
\begin{lemma}
\label{second_inequality_lemma}
The following inequality holds:
\begin{gather}
\label{second_inequality}
R(\widehat{f}_{\mathrm{pPC}}) - R(f^\star)\leq \frac{2}{1-\rho_+}\sup_{f\in\mathcal{F}}\left|R_{\mathrm{pPC}}^+(f) - \widehat{R}^+_{\mathrm{pPC}}(f)\right| + \frac{2}{1-\rho_-}\sup_{f\in\mathcal{F}}\left|R_{\mathrm{pPC}}^-(f) - \widehat{R}_{\mathrm{pPC}}^-(f)\right|.
\end{gather}
\end{lemma}
\begin{proof}
We omit the proof of Lemma \ref{second_inequality_lemma} since it is quite similar to that of Lemma \ref{inequality_lemma}.
\end{proof}
As suggested by Lemma \ref{second_inequality_lemma}, we need to further upper bound the right hand size of Eq.~(\ref{second_inequality}). According to Lemma \ref{uniform_bound}, we have the following two lemmas.
\begin{lemma}
\label{lemma_of_third_inequality}
Suppose the loss function $\ell$ is $\rho$-Lipschitz with respect to the first argument ($0<\rho<\infty$), and all the functions in the model class $\mathcal{F}$ are bounded, i.e., there exists a constant $C_{\mathrm{b}}$ such that $\left\|f\right\|_{\infty}\leq C_{\mathrm{b}}$ for any $f\in\mathcal{F}$. Let $C_{\ell}:=\sup_{z\leq C_{\mathrm{b}},t=\pm 1}\ell(z,t)$. For any $\delta>0$, with probability $1-\delta$,
\begin{gather}
\nonumber
\sup_{f\in\mathcal{F}}\left|R_{\mathrm{pPC}}^+(f) - \widehat{R}^+_{\mathrm{pPC}}(f)\right|\leq 2\rho\widetilde{\mathfrak{R}}_n^+(\mathcal{F}) + C_{\ell}\sqrt{\frac{\log\frac{2}{\delta}}{2n}}.
\end{gather}
\end{lemma}
\begin{lemma}
\label{lemma_of_fourth_inequality}
Suppose the loss function $\ell$ is $\rho$-Lipschitz with respect to the first argument ($0<\rho<\infty$), and all the functions in the model class $\mathcal{F}$ are bounded, i.e., there exists a constant $C_{\mathrm{b}}$ such that $\left\|f\right\|_{\infty}\leq C_{\mathrm{b}}$ for any $f\in\mathcal{F}$. Let $C_{\ell}:=\sup_{z\leq C_{\mathrm{b}}, t=\pm 1}\ell(z,t)$. For any $\delta>0$, with probability $1-\delta$,
\begin{gather}
\nonumber
\sup_{f\in\mathcal{F}}\left|R_{\mathrm{pPC}}^-(f) - \widehat{R}^-_{\mathrm{pPC}}(f)\right|\leq 2\rho\widetilde{\mathfrak{R}}_n^-(\mathcal{F}) + C_{\ell}\sqrt{\frac{\log\frac{2}{\delta}}{2n}}.
\end{gather}
\end{lemma}
We omit the proofs of Lemma \ref{lemma_of_third_inequality} and Lemma \ref{lemma_of_fourth_inequality} since they are similar to that of Lemma \ref{lemma_of_first_inequality}.

By combing Lemma \ref{second_inequality_lemma}, Lemma \ref{lemma_of_third_inequality}, and Lemma \ref{lemma_of_fourth_inequality}, Theorem \ref{second_estimation_error} is proved.
\section{Supplementary Information of Experiments}\label{experiments_supp}
Table \ref{datasets} reports the specification of the used benchmark datasets and models.

\noindent\textbf{MNIST}\footnote{\url{http://yann.lecun.com/exdb/mnist/}} \cite{lecun1998gradient}.\quad This is a grayscale image dataset composed of handwritten digits from 0 to 9 where the size of the each image is $28\times 28$. It contains 60,000 training images and 10,000 test images. Because the original dataset has 10 classes, we regard the even digits as the positive class and the odd digits as the negative class.

\noindent\textbf{Fashion-MNIST}\footnote{\url{https://github.com/zalandoresearch/fashion-mnist}} \cite{xiao2017fashion}.\quad Similarly to MNIST, this is also a grayscale image dataset composed of fashion items (`T-shirt', `trouser', `pullover', `dress', `sandal', `coat', `shirt', `sneaker', `bag', and `ankle boot'). It contains 60,000 training examples and 10,000 test examples. It is converted into a binary classification dataset as follows:
\begin{itemize}
\item The positive class is formed by `T-shirt', `pullover', `coat', `shirt', and `bag'.
\item The negative class is formed by `trouser', `dress', `sandal', `sneaker', and `ankle boot'.
\end{itemize}
\noindent\textbf{Kuzushiji-MNIST}\footnote{\url{https://github.com/rois-codh/kmnist}} \cite{netzer2011reading}.\quad This is another grayscale image dataset that is similar to MNIST. It is a 10-class dataset of cursive Japanese (``Kuzushiji") characters. It consists of 60,000 training images and 10,000 test images. It is converted into a binary classification dataset as follows:
\begin{itemize}
\item The positive class is formed by `o', `su',`na', `ma', `re'.
\item The negative class is formed by `ki',`tsu',`ha', `ya',`wo'.
\end{itemize}
\noindent\textbf{CIFAR-10}\footnote{\url{https://www.cs.toronto.edu/~kriz/cifar.html}} \cite{krizhevsky2009learning}.\quad This is also a color image dataset of 10 different objects (`airplane', `bird', `automobile', `cat', `deer', `dog', `frog', `horse', `ship', and `truck'), where the size of each image is $32\times 32\times 3$. There are 5,000 training images and 1,000 test images per class. This dataset is converted into a binary classification dataset as follows:
\begin{itemize}
\item The positive class is formed by `bird', `deer', `dog', `frog', `cat', and `horse'. 
\item The negative class is formed by `airplane', `automobile', `ship', and `truck'. 
\end{itemize}
\noindent\textbf{USPS, Pendigits, Optdigits.}\quad These datasets are composed of handwritten digits from 0 to 9. Because each of the original datasets has 10 classes, we regard the even digits as the positive class and the odd digits as the negative class.

\noindent\textbf{CNAE-9.}\quad This dataset contains 1,080 documents of free text business descriptions of Brazilian companies categorized into a subset of 9 categories cataloged in a table called National Classification of Economic Activities.
\begin{itemize}
\item The positive class is formed by `2', `4', `6' and `8'. 
\item The negative class is formed by `1', `3', `5', `7' and `9'.
\end{itemize}
For MNIST, Kuzushiji-MNIST, and Fashion-MNIST, we set learning rate to $1e-3$ and weight decay to $1e-5$. For CIFAR-10, we set learning rate to $1e-3$ and weight decay to $1e-3$.
We also list the number of pointwise corrupted examples used for model training on each dataset: 30,000 for MNIST, Kuzushiji-MNIST, and Fashion-MNIST; 20,000 for CIFAR-10;
4,000 for USPS; 5,000 for Pendigits; 2,000 for Optdigits; 400 for CNAE-9.

\end{document}